\documentclass[10pt]{article} 
\usepackage{arxiv}

\usepackage{amsmath,amsfonts,bm}

\def\eqref#1{equation~\ref{#1}}

\def\1{\bm{1}}

\DeclareMathAlphabet{\mathsfit}{\encodingdefault}{\sfdefault}{m}{sl}
\SetMathAlphabet{\mathsfit}{bold}{\encodingdefault}{\sfdefault}{bx}{n}

\newcommand{\E}{\mathbb{E}}

\newcommand{\R}{R}

\newcommand{\KL}{\mathrm{D}_{\mathrm{KL}}}

\DeclareMathOperator*{\argmax}{arg\,max}

\RequirePackage{xspace}

\newcommand{\ie}{\textit{i.e.}\xspace}

\newcommand{\eg}{\textit{e.g.}\xspace}

\newcommand{\iid}{\textit{i.i.d.}\xspace}

\DeclareMathOperator*{\Esp}{\mathbb{E}}

\newcommand{\Dcal}{{\mathcal{D}}}

\newcommand{\Mcal}{{\mathcal{M}}}

\newcommand{\Wcal}{{\mathcal{W}}}
\newcommand{\Xcal}{{\mathcal{X}}}
\newcommand{\Ycal}{{\mathcal{Y}}}

\newcommand{\Rbb}{{\mathbb{R}}}

\newcommand{\kl}{{\textup{kl}}}
\newcommand{\MV}{{\textup{MV}}}

\newcommand{\I}{{\textup{I}}}

\DeclareMathOperator*{\Prob}{\mathbb{P}}

\DeclareMathOperator{\balpha}{\boldsymbol{\alpha}}
\DeclareMathOperator{\bbeta}{\boldsymbol{\beta}}
\DeclareMathOperator{\brho}{\boldsymbol{\rho}}
\DeclareMathOperator{\bpi}{\boldsymbol{\pi}}

\DeclareMathOperator{\A}{\mathcal{A}}
\DeclareMathOperator{\D}{\Dcal}
\DeclareMathOperator{\Hcal}{\mathcal{H}}
\DeclareMathOperator{\jointerr}{\mathcal{E}}
\DeclareMathOperator{\dis}{\mathfrak{D}}
\DeclareMathOperator{\cbound}{\mathcal{C}}

\newcommand{\wrongweight}{w_{\post}}
\newcommand{\Sone}{S_1}
\newcommand{\Stwo}{S_2}
\newcommand{\none}{{n_1}}
\newcommand{\ntwo}{{n_2}}
\newcommand{\Hcalone}{\Hcal_1}
\newcommand{\Hcaltwo}{\Hcal_2}
\newcommand{\Wcalone}{\Wcal_1}
\newcommand{\Wcaltwo}{\Wcal_2}
\newcommand{\hyperpriorone}{\hyperprior_1}
\newcommand{\hyperpriortwo}{\hyperprior_2}

\newcommand{\posteriorone}{\posterior_{1}}
\newcommand{\posteriortwo}{\posterior_{2}}

\newcommand{\Aone}{A_{1}}
\newcommand{\Atwo}{A_{2}}

\newcommand{\prior}{\bpi}
\newcommand{\post}{\brho}
\newcommand{\posterior}{\post}

\newcommand{\hyperprior}{P}
\newcommand{\hyperposterior}{Q}

\newcommand{\Risk}{R}
\newcommand{\Riskemp}{\widehat{\Risk}_{S}}
\newcommand{\Risktrue}{\Risk_{\D}}

\newcommand{\gap}{\Psi}

\DeclareMathOperator{\Dir}{\operatorname{Dir}}

\DeclareMathOperator{\DivRen}{\mathrm{D}_\lambda}

\newcommand{\complexity}{\epsilon}

\usepackage{url}
\usepackage{graphicx} 
\let\AND\relax
\usepackage{algorithm}
\usepackage{algorithmic}
\usepackage{amsmath,amssymb, amsthm}
\usepackage{mathtools}
\usepackage{lscape}
\usepackage{makecell}
\usepackage[%
pdfstartview=FitH,%
breaklinks=true,%
bookmarks=false,%
colorlinks=true,%
linkcolor= blue,
anchorcolor=blue,%
citecolor=blue,
filecolor=blue,%
menucolor=blue,%
urlcolor=blue%
]{hyperref}
\usepackage{multicol,lipsum}
\usepackage{pdflscape}
\usepackage[dvipsnames]{xcolor}
\usepackage[english]{babel}
\usepackage{setspace}
\usepackage{bbm}
\usepackage{tikz}
\usepackage{mathrsfs}
\usepackage{amssymb}
\usepackage{nicematrix}
\usepackage{fancyhdr}
\usepackage{lipsum}
\usepackage{subcaption}
\usepackage{rotating}
\usepackage{cleveref}
\usepackage{thmtools} 
\usepackage{thm-restate}
\usepackage{nicefrac}
\usepackage{enumerate}
\usepackage{natbib}

\definecolor{blue}{HTML}{0077BB}
\definecolor{cyan}{HTML}{33BBEE}
\definecolor{green}{HTML}{009988}
\definecolor{orange}{HTML}{EE7733}
\definecolor{red}{HTML}{CC3311}
\definecolor{magenta}{HTML}{EE3377}
\definecolor{grey}{HTML}{BBBBBB}
\definecolor{pink}{HTML}{FF007F}

\definecolor{ao}{rgb}{0.0, 0.5, 0.0}
\newtheorem{theorem}{Theorem}
\newtheorem{corollary}{Corollary}

\usepackage{{booktabs}}

\title{On the Disintegration of the Stochastic Majority Vote: 
From PAC-Bayesian Bounds to a Self-Bounding Algorithm}

\author{Julien Bastian\\
Université Jean Monnet Saint-Étienne, CNRS, Institut d Optique Graduate School,\\
Laboratoire Hubert Curien UMR 5516, F-42023, Saint-Etienne, France\\
\texttt{julien.bastian@univ-st-etienne.fr}
\AND 
Benjamin Leblanc, \hspace*{5mm} Pascal Germain\\
Département d'informatique et de génie logiciel, Université Laval, Québec, Canada\\
\texttt{benjamin.leblanc.2@ulaval.ca}\hspace*{5mm} \texttt{pascal.germain@ulaval.ca}
\AND 
Amaury Habrard \\
Université Jean Monnet Saint-Étienne, CNRS, Institut d Optique Graduate School,\\
Laboratoire Hubert Curien UMR 5516, Inria, F-42023, Saint-Etienne, France\\
Institut Universitaire de France\\
\texttt{amaury.habrard@univ-st-etienne.fr}
\AND
Guillaume Metzler\\
Université Lumière Lyon 2, Universite Claude Bernard Lyon 1, ERIC, 69007, Lyon, France \\
\texttt{guillaume.metzler@univ-lyon2.fr}
\AND 
Emilie Morvant \\
Université Jean Monnet Saint-Étienne, CNRS, Institut d Optique Graduate School,\\
Laboratoire Hubert Curien UMR 5516, F-42023, Saint-Etienne, France\\
\texttt{emilie.morvant@univ-st-etienne.fr}
\AND 
Paul Viallard\\
Univ Rennes, Inria, CNRS IRISA - UMR 6074, F35000 Rennes, France\\
\texttt{paul.viallard@inria.fr}
}

\begin{document}
\maketitle

\begin{abstract}
\looseness=-1
Weighted majority votes are central to many successful ensemble methods.
PAC-Bayesian theory provides tight generalization guarantees for such models by analyzing the expected risk of stochastic classifiers, while analyzing the risk of deterministic majority votes relies on surrogate bounds. 
To avoid these surrogates, \citet{zantedeschi2021learning} introduced guarantees for stochastic majority votes, but the resulting models remain randomized.
In this paper, we propose a derandomization framework for stochastic majority votes. 
To do so, we apply recent advances in disintegrated PAC-Bayesian theory directly to the space of majority vote weight vectors,  
transforming stochastic guarantees into certificates for a single deterministic majority vote. 
We derive two families of high-probability generalization bounds, covering both data-independent and data-dependent constructions of the ensemble, which naturally lead to a self-bounding learning algorithm optimizing deterministic majority vote guarantees.
\end{abstract}

\section{Introduction}
\looseness=-1
Ensemble methods \citep{dietterich2000ensemble} based on weighted majority votes are among the most classical and widely used approaches in machine learning, especially when dealing with weak base learners (also called voters).
Their empirical success comes from their ability to exploit voter diversity by aggregating them into a single decision rule
\citep[see, \eg,][]{kuncheva2004combining}, such as for Bagging \citep{breiman1996bagging}, random forest \citep{breiman2001random}, boosting \citep{freund1996experiments}.

\looseness=-1
The PAC-Bayesian statistical learning theory \citep{shawetaylor1997pac,mcallester1998pac} provides a principled framework for analyzing generalization in ensemble methods.
The core idea is to consider a posterior (weights) distribution over a set of voters and to provide generalization guarantees for the weighted-average risk of these voters \citep[\eg,][]{mcallester1998pac,seeger2002pac,catoni2007pac}.
A particularly appealing characteristic of PAC-Bayesian bounds is that they can directly lead to learning algorithms.
This has been explored through the notion of self-bounding algorithms of \citet{freund1998self} 
\citep[\eg,][]{langford2003microchoice,ambroladze2006tighter,dziugaite2017computing,germain2009pac,viallard2021self,viallard2024general,atbir2026pac}, where the training objective explicitly minimizes a generalization bound. 
Classical PAC-Bayesian bounds control the weighted-average risk of the voters, which coincides with the expected risk of the stochastic classifier associated with the posterior distribution (often called the Gibbs classifier).
However, since practitioners are usually interested in a single deterministic classifier, an important line of research has focused on transferring PAC-Bayesian bounds from the stochastic classifier to the deterministic weighted majority vote induced by the same posterior distribution \citep[see, \eg,][]{langford2002pac,lacasse2006pac,laviolette2011pac,germain2015risk,masegosa2020second,viallard2021self}.
A common strategy is to upper-bound the vote's risk using surrogate quantities relying on the stochastic classifier's risk.
The simplest example is the ``factor-two'' bound~\citep[stating that the majority vote's risk is upper-bounded by twice the stochastic classifier's risk, see ][]{langford2002pac}, while tighter analyses have been obtained through the binomial bound \citep{shawetaylor2009pac,lacasse2010learning}, the second-order bound~\citep{masegosa2020second} or the C-bound~\citep{breiman2001random,lacasse2006pac,laviolette2011pac,germain2015risk,laviolette2017risk,viallard2021self}.
Even if these approaches have led to PAC-Bayesian self-bounding algorithms~\citep[\eg,][]{masegosa2020second,viallard2021self}, the surrogate usually does not directly account for the risk of the deterministic majority vote, leading to looser generalization guarantees.

\looseness=-1
To tackle this, \citet{zantedeschi2021learning} proposed a stochastic majority vote approach, where the ensemble itself is sampled from a probability distribution. 
This introduces a second level of randomization: instead of sampling a single voter from a posterior distribution, one samples a whole majority vote, that is, an entire set of weights.
In particular, modeling the distribution over weight vectors with a Dirichlet distribution provides a flexible probabilistic model on the simplex (non-negative values summing to one) while preserving tractability. 
Remarkably, in binary classification, the expected $0$-$1$ loss of the resulting stochastic majority vote admits a closed-form expression, enabling the direct optimization of the associated PAC-Bayesian bound. 
This leads to a self-bounding learning algorithm that directly minimizes the bound with tight certificates. 
Despite this advantage, 
the learned predictor remains stochastic: a new majority vote must be sampled at each prediction, leading to variability and making the deployment difficult.

\looseness=-1
In this paper, we extend the stochastic majority vote paradigm of \citet{zantedeschi2021learning} to \textit{disintegrated} PAC-Bayesian bounds, that is, bounds that control the risk of a single model drawn from the posterior distribution \citep{catoni2007pac,blanchard2007occam}. 
More specifically, we build on two recent advances in this line of research \citep{rivasplata2020pac,viallard2024general} to derive bounds for deterministic majority vote predictors that are optimizable.
This allows us to first transform PAC-Bayesian guarantees for a stochastic model into guarantees for a deterministic one,
and then to design corresponding self-bounding learning algorithms that directly optimize these guarantees.
Experimental results on classification tasks show that our approach yields deterministic majority votes together with competitive PAC-Bayesian guarantees.
Compared to stochastic majority votes, our method \textit{(i)} produces deterministic majority votes, \textit{(ii)} requires less training time, and \textit{(iii)} remains competitive in terms of accuracy and certificate tightness. 
Moreover, it substantially improves upon classical PAC-Bayesian majority vote bounds based on surrogate analyses. 
Our work provides a principled bridge between stochastic majority vote learning and deterministic certification.

\textbf{Organization of the paper.}
\Cref{sec:PB} gives basics on classical and disintegrated PAC-Bayes, and \Cref{sec:PBMV} recalls PAC-Bayes analyses of deterministic and stochastic majority vote.
\Cref{sec:contribution} states our contribution, from which we derive a self-bounding algorithm in \Cref{sec:algo}, empirically evaluated in \Cref{sec:expe}.

\section{Basics on PAC-Bayesian Theory}
\label{sec:PB}
\subsection{Classical PAC-Bayesian Bounds}\label{subsec:classic_pb}
\looseness=-1
We focus on 
classification tasks from an input space $\Xcal$ to an output space $ \Ycal$.
We denote by~$\Dcal$ the unknown data-generating distribution over $\Xcal{\times}\Ycal$. 
A learning set $S=\{(x_i,y_i)\}_{i=1}^n$ consists of~$n$ examples drawn \iid from $\Dcal$; we denote by $\Dcal^n$ the distribution of such an $n$-sample.
We consider a finite\footnote{We restrict $\Hcal$ to be finite as required by the majority vote framework we study in this paper; note that this restriction is not needed for classical PAC-Bayesian bounds, which can be formulated for continuous hypothesis spaces.} set of hypotheses/voters $\Hcal=\{h_1,\dots,h_{|\!\Hcal\!|}\}$, 
where each $h\!\in\! \Hcal$ maps $\Xcal$ to $\Ycal$.
Given $\ell:\widehat{\Ycal}{\times}\Ycal\to[0,1]$ a loss function, where $\widehat{\Ycal}$ denotes the prediction space, the \textit{true risk} on $\Dcal$ and \textit{empirical risk} on $S$ of a hypothesis $h\in\Hcal$ respectively are
\begin{align*}
\Risktrue(h) = \Esp_{(x,y)\sim\Dcal}\,\ell(h(x),y),\quad\quad\mbox{and}\quad\quad \Riskemp(h) = \frac{1}{n}\sum_{i=1}^n \ell(h(x_i),y_i).
\end{align*}
Since $\Dcal$ is unknown, statistical learning theory provides tools to estimate how close the empirical risk~$\Riskemp(h)$ is to the true risk~$\Risktrue(h)$.
In particular, Probably Approximately Correct (PAC) learning theory \citep{valiant1984theory}
provides high-probability generalization bounds that relate the true risk to the empirical risk.
Such guarantees are usually expressed through a deviation function
$\gap(\Riskemp(h),\Risktrue(h))$, often referred to as the generalization gap.
For instance, when $\gap(a,b)=|a-b|$, a PAC bound takes the  general form
\begin{align*}
\Prob_{S\sim\Dcal^n}\left[\,\left|\Riskemp(h) - \Risktrue(h) \right| \leq  \complexity(h,n,\delta)\, \right]\geq 1-\delta,\quad \text{for some positive function $\complexity$}.
\end{align*}
Put into words, with high probability over the random draw of the learning set 
$S$ of size~$n$, a hypothesis will exhibit good generalization guarantees when the generalization gap between the true and the empirical risk is low, \ie, the value of $\complexity(h,n,\delta)$ should be as small as possible.

\looseness=-1
We focus on PAC-Bayesian generalization bounds, introduced by \citet{shawetaylor1997pac,mcallester1998pac}.
Unlike classical PAC bounds, which provide guarantees for a fixed hypothesis, PAC-Bayes studies the risk of the stochastic (Gibbs) classifier obtained by sampling, for each data point, a hypothesis from a learned distribution over~$\Hcal$.
Let $\Mcal$ denote the set of probability distributions over its input and $\Mcal^*$ the set of strictly positive distributions. 
PAC-Bayes relies on two distributions over $\Hcal$: a \textit{prior} distribution \mbox{$\prior\!\in\!\Mcal^*(\Hcal)$}, encoding prior knowledge before observing 
$S$, and a learned \textit{posterior} distribution \mbox{$\posterior\!\in\!\Mcal(\Hcal)$}.
The goal is to upper-bound the true risk of the stochastic classifier $\Esp_{h\sim \posterior}\Risktrue(h)$ through the generalization gap $\gap\big(\Esp_{h\sim \posterior}\Riskemp(h), \Esp_{h\sim \posterior}\Risktrue(h)\big)$.
We recall below the general PAC-Bayesian theorem of \citet{germain2009pac}.
\begin{theorem}[\small General PAC-Bayesian Theorem of \citet{germain2009pac}, modified]
\label{theo:PB-general}
\looseness=-1
   For any distribution $\Dcal$, hypothesis set $\Hcal$, prior distribution $\prior\in\Mcal^*(\Hcal)$, measurable function $\varphi:\Hcal\times(\Xcal{\times}\Ycal)^n\to\mathbb{R}$, and $\delta\in (0,1]$, we have 
$$
\Prob_{S\sim \Dcal^n}\left[
\forall \post\in\Mcal(\Hcal),\ \
\Esp_{h\sim\post} \varphi(h,S) \leq \KL(\posterior\|\prior) + \ln\left(\frac1\delta \Esp_{S'\sim\Dcal^n} \Esp_{h'\sim\prior} e^{\varphi(h',S')}\right) 
\right]\geq 1-\delta,
$$
where $\KL(\posterior\|\prior)=\Esp_{h\sim \posterior} \ln \frac{\posterior(h)}{\prior(h)} $ is the Kullback-Leibler divergence between $\posterior$ and $\prior$.
\end{theorem}
\looseness=-1
\Cref{theo:PB-general} holds for any posterior $\post\!\in\!\Mcal(\Hcal)$ and is penalized by the KL divergence between $\posterior$ and $\prior$: the closer $\posterior$ is to $\prior$, the tighter the resulting guarantee.
Importantly, from this generic result, one can recover many classical PAC-Bayes bounds through an appropriate choice of $\varphi$ \citep[\eg,][]{mcallester1998pac,seeger2002pac,catoni2007pac,maurer2004note}.
For example, choosing $\varphi(h,S)\!\triangleq\!n\gap(\Riskemp(h),\Risktrue(h))$ (where $\gap$ is convex) leads to bounds controlling the generalization gap between empirical and true risks.
%
We now recall its instantiation 
with $\varphi(h,S) \!\triangleq\! n\kl(\Riskemp(h)\|\Risktrue(h))$, where $\kl(p\|q)\!
\triangleq\! p \ln\frac{p}{q}
+ (1-p)\ln\frac{1-p}{1-q}$ is
the KL divergence between two Bernoulli distributions with parameters
$p$ and~$q$.
\begin{theorem}[\small \citet{seeger2002pac,maurer2004note}]
\label{thm:classical_kl}
   For any distribution $\Dcal$, hypothesis set $\Hcal$, prior distribution~\mbox{$\prior\!\in\!\Mcal^*(\Hcal)$}, and $\delta\in (0,1]$, we have
$$
\Prob_{S\sim \Dcal^n}\left[
\forall \posterior\in\Mcal(\Hcal),\ \
\kl\left(\Esp_{h \sim \posterior}\Riskemp(h)\Big\|\Esp_{h \sim \posterior}\Risktrue(h)\right)\leq \frac{1}{n}\left[\KL(\posterior\|\prior)+\ln\left(\frac{2\sqrt{n}}{\delta}\right)\right]\,
\right]\geq 1-\delta.
$$
\end{theorem}
\looseness=-1
PAC-Bayesian bounds are particularly appealing since \textit{(i)} their data-dependent nature generally leads to tight and computable generalization bounds, and \textit{(ii)} their ability to lead to 
 algorithms that directly minimize the bound gives rise to \textit{self-bounding} algorithms \citep{freund1998self}.

\subsection{Disintegrated PAC-Bayesian Bounds}
\label{sec:disintegration}
\looseness=-1
A limitation of classical PAC-Bayesian bounds is that they provide guarantees in expectation with respect to the posterior distribution.
Consequently, they certify the performance of a stochastic classifier rather than that of a single deterministic one, the latter being often required in practice.
For instance, in medical diagnosis \citep{naji2021machine}, a patient must receive a single and reproducible diagnosis: randomly changing decisions would be undesirable.
In response to such a lack in classical PAC-Bayesian bounds, another line of research known as \textit{disintegrated} PAC-Bayes bounds was initiated by \citet{catoni2007pac,blanchard2007occam}. 
Instead of bounding the risk of a stochastic classifier with respect to the posterior distribution, disintegrated PAC-Bayesian bounds upper-bound the risk of a single classifier sampled from a posterior distribution $\posterior_S$ learned by an algorithm $A:(\Xcal{\times}\Ycal)^n\times\Mcal^*(\Hcal)\to\Mcal(\Hcal)$ involving both a dataset and a prior distribution. 
We recall below the general disintegrated PAC-Bayesian bound proposed by \citet{rivasplata2020pac} that is a derandomization of the general PAC-Bayesian \Cref{theo:PB-general}.
\begin{restatable}[\small General Disintegrated PAC-Bayesian Theorem of~\citet{rivasplata2020pac}]{theorem}{GeneralDisintegratedRivasplate}
\label{thm:rivasplata}
\looseness=-1
For any distribution~$\Dcal$, hypothesis set~$\Hcal$, prior distribution $\prior \!\in\! \Mcal^*(\Hcal)$, measurable function \mbox{$\varphi: \Hcal {\times}(\Xcal{\times}\Ycal)^n \rightarrow \mathbb{R}$}, algorithm $A:(\Xcal{\times}\Ycal)^n \times \Mcal^*(\Hcal) \rightarrow \Mcal(\Hcal)$, and $\delta \in(0,1]$, we have
\begin{align*}
\underset{\substack{S \sim \Dcal^n\\h \sim \post_{S}}}{\mathbb{P}}\!
\left[\varphi(h, S) \leq \ln \left[\frac{\post_{S}(h)}{\prior(h)}\right]+\ln \left[\frac{1}{\delta} \underset{S^{\prime} \sim \Dcal^n}{\E} \underset{h^{\prime} \sim \prior}{\E} \exp \left(\varphi\left(h^{\prime}, S^{\prime}\right)\right)\right]\right] \geq 1\!-\!\delta, \quad \text{with}\quad  \post_{S} \triangleq A(S, \prior).
\end{align*}
\end{restatable}
\looseness=-1
Compared to Theorem~\ref{theo:PB-general}, the expectation over the posterior 
is moved outside the probability statement.
Then, the bound holds with high probability for a single model $h$ sampled from the learned posterior~$\posterior_S$.
Moreover, the KL divergence is replaced by its ``disintegrated'' version that only involves the density ratio between $\posterior_S(h)$ and~$\prior(h)$.
More recently, \citet{viallard2024general} proposed the following alternative theorem, where the penalty is expressed through the Rényi divergence between $\posterior_S$ and $\prior$, and thus depends on the entire $\Hcal$.
\begin{restatable}[\small General Disintegrated PAC-Bayesian Theorem of~\citet{viallard2024general}]{theorem}{GeneralDisintegratedViallard}
\label{thm:viallard}
\looseness=-1
Under the same assumptions as \Cref{thm:rivasplata}, assuming additionally that \mbox{$\varphi:\Hcal\times(\Xcal{\times}\Ycal)^n\rightarrow\mathbb{R}_+^*$}
and $\lambda>1$, we have
\begin{align*}
\underset{\substack{S \sim \Dcal^n\\h \sim \post_S}}{\mathbb{P}}
\left[\, \frac{\lambda}{\lambda\!-\!1} \ln (\varphi(h, S))\leq \frac{2 \lambda\!-\!1}{\lambda\!-\!1} \ln \frac{2}{\delta}+\DivRen\!\left(\post_S \| \prior\right)+\ln \left[\underset{S^{\prime} \sim \Dcal^n}{\E} \underset{h^{\prime} \sim \prior}{\E}\left(\varphi\left(h^{\prime}, S^{\prime}\right)^{\frac{\lambda}{\lambda\!-\!1}}\right)\right]\,\right] \geq 1-\delta,
\end{align*}
where $\post_{S} \triangleq A(S, \prior)$, and $\DivRen\left(\post_S \| \prior\right)\triangleq {\displaystyle \frac{1}{\lambda\!-\!1}} \ln \left[\Esp_{h\sim\prior} \left[\frac{\post_S(h)}{\prior(h)}\right]^\lambda\right]$ is  the Rényi divergence. 
\end{restatable}
\looseness=-1
These results motivate a learning procedure:
\textit{(i)} learn a posterior $\post_S$ from $S$ using an algorithm $A$;
\textit{(ii)} sample a single model $h\!\sim\!\posterior_S$ for deployment; 
and
\textit{(iii)} obtain a high-probability generalization guarantee for this specific deterministic $h$.
In the next section, we recall another common approach in PAC-Bayes to get bounds to study the risk of the deterministic majority vote induced by the posterior distribution $\posterior$. 

\section{PAC-Bayes for Majority Vote}
\label{sec:PBMV}

\subsection{The Deterministic Weighted Majority Vote}
\label{sec:detMV}
\looseness=-1
In practice, the stochastic Gibbs classifier associated with a posterior $\post$ is rarely used directly at prediction time.
Instead, one can aggregate the predictions of the voters in $\Hcal$ according to their posterior weights, yielding a deterministic \textit{weighted majority vote}, denoted by $\MV_{\posterior}$. 
Given a posterior distribution $\post$ over $\Hcal$, the weighted majority vote predicts the label with the largest total posterior weight, it is defined by
\begin{align}
    \label{eq:mv}
    \MV_{\posterior}(x) =  \argmax_{y\in\Ycal}     \left\{ \Esp_{h\sim \post}\,\I [h(x) = y] \right\},
\end{align}
where $\I [a]\!=\!1$ if $a$ is true and $0$ otherwise. 
Its \textit{true risk} and \textit{empirical risk} are, respectively,
\begin{align}
\label{eq:mv-risks}
    \Risktrue (\MV_{\post}) = \Esp_{(x,y)\sim\Dcal} \ell \left(\MV_{\post}(x),y\right),\quad \mbox{and}\quad \Riskemp (\MV_{\post}) = \frac{1}{n}\sum_{i=1}^n \ell \left(\MV_{\post}(x_i),y_i\right).
\end{align}
\looseness=-1
Majority votes play a central role in PAC-Bayes, as they often achieve good empirical performance while still benefiting from PAC-Bayesian guarantees through their connection with the stochastic Gibbs classifier.
However, the challenge is that the majority vote involves a non-differentiable aggregation of the voters, making its risk hard to analyze directly.
For this reason, a part of the literature has focused on the specialization of PAC-Bayes bounds to majority vote.
One of the most classical approaches consists in upper-bounding the risk of the majority vote with surrogate quantities that depend on the risk of the stochastic classifier~\citep[\eg,][]{langford2002pac,shawetaylor2009pac,lacasse2010learning,lacasse2006pac,masegosa2020second}.
As a consequence, most PAC-Bayesian analyses do not directly control the generalization gap for the risk of the majority vote itself.
Instead, they control surrogate quantities from which guarantees on the majority vote can be derived with varying degrees of tightness.
We recall below the classical surrogates instantiated with the $0$-$1$ loss,
defined by $\ell(h(x),y) = \I[h(x)\neq y]$.

\looseness=-1
\textbf{Factor-two bound}  \citep[or first-order bound,
][]{langford2002pac}\textbf{.}
Using the first-order Markov inequality, the risk of the majority vote satisfies
\begin{align}
\label{eq:factor-two}
\Risktrue(\MV_{\post}) \leq 2 \,\Esp_{h\sim \post}\Risktrue(h).
\end{align} 
This inequality relates the deterministic majority vote risk to the risk $\E_{h\sim\post}\Risktrue(h)$.
Despite its simplicity, this surrogate lacks precision, since \textit{(i)} when~$\E_{h\sim \post}\Risktrue(h)$ exceeds $\nicefrac{1}{2}$, the bound exceeds $1$ and becomes uninformative, and \textit{(ii)} the surrogate is close to $0$ only if the risk of the stochastic classifier is close to $0$ too.
In particular, it does not exploit the diversity or error correlations among voters in $\Hcal$, which is key when learning with a majority vote: if all the voters were always right, then 
one does not need to combine them. 

\looseness=-1
\textbf{Binomial bound} \citep{shawetaylor2009pac,lacasse2010learning}\textbf{.}
It estimates the majority vote risk by sampling $N$ base voters and computing the probability that at least half of them make an error:
\begin{equation*}
 b_{\Dcal}^N(\post)\triangleq 
\Esp_{(x,y)\sim\Dcal}\left[\,
\sum_{j=\lceil\frac{N}{2}\rceil}^N\binom{N}{j} \wrongweight(x,y)^j \left(1-\wrongweight(x,y)\right)^{(N-j)}\right],   
\end{equation*}
with $\wrongweight(x,y)\triangleq \Esp_{h\sim\posterior} \I[h(x)\!\ne\! y]$ the posterior weight assigned to misclassifying voters.
As $N$ increases, $b_{\Dcal}^N$ provides a tighter approximation of the majority vote risk, but the corresponding PAC-Bayes bound becomes looser since its penalty term scales linearly with $N$.
The resulting surrogate relies on a factor-two inequality: 
\begin{align}
    \label{eq:bin-bound}
    \Risktrue(\MV_{\post}) \leq 2\,  b_\Dcal^N(\posterior).
\end{align}

\looseness=-1
\textbf{Second-order bound} \citep{masegosa2020second}\textbf{.}
One solution to account for the diversity of voters with slightly greater precision is to use the joint error defined by
\begin{align*}
 \jointerr_{\Dcal}(\post) = \Esp_{(x,y)\sim \Dcal}\Esp_{(h,h')\sim\post^2} \I[h(x)\ne y]\,\I[h'(x)\ne y].
\end{align*}
More precisely, by applying second-order Markov's inequality, we have 
\begin{align}
 \label{eq:second-order-bound}
 \Risktrue(\MV_{\post}) \leq 4\,\jointerr_{\Dcal}(\post).
\end{align}
This surrogate reflects the idea that, for a majority vote to perform well, voters must be sufficiently diverse, and if they make mistakes, those mistakes must occur for different predictions.
However, if the joint error $\jointerr_{\Dcal}(\post)$ exceeds $\nicefrac{1}{4}$, the surrogate exceeds $1$ and is uninformative.

\looseness=-1
\textbf{C-bound}\footnote{The term ``C-bound'' is specific to PAC-Bayes.} \citep{lacasse2006pac}\textbf{.}
This bound involves both the joint error and the disagreement between voters, allowing the diversity/complementarity of voters to be taken into account explicitly.
It follows from the Cantelli-Chebyshev inequality and is expressed as
\begin{align*}
\mbox{if}\ \Esp_{h\sim\post} \Risktrue(h) \leq \frac{1}{2}\quad \mbox{we have}\quad  
\Risktrue(\MV_{\post}) \le 1 - \frac{\big(1-(2\jointerr_{\Dcal}(\post)+\dis_{\Dcal}(\post))\big)^2}{1-2\dis_{\Dcal}(\post)} \triangleq \cbound_{\Dcal}(\post),
\end{align*}
where $\displaystyle \dis_{\Dcal}(\post) = \Esp_{(x,y)\sim\Dcal} \Esp_{(h,h')\sim\post^2} \I[h(x)\ne h'(x)]$
is the disagreement.

\looseness=-1
\textbf{Relationships between the surrogates.}
For any distribution $\Dcal$, for any voter set $\Hcal$, for any distribution~$\posterior$ on $\Hcal$, if $\Esp_{h\sim\posterior} \Risktrue(h) < \frac{1}{2}$, we have
\begin{align*}
&\textit{(i)}~~\Risktrue(\MV_{\post})\  \le \ \cbound_{\Dcal}(\post)\ \le\ 4\jointerr_{\Dcal}(\post)\ \le \ 2\Esp_{h\sim\post} \Risktrue(h), \quad \text{if  }\Esp_{h\sim\post} \Risktrue(h) \le  \dis_{\Dcal}(\post),\\
&\textit{(ii)}~\Risktrue(\MV_{\post})\ \le\ 2 \Esp_{h\sim\post} \Risktrue(h)\ \le\ \cbound_{\Dcal}(\post)\ \le\ 4\jointerr_{\Dcal}(\post),\quad \text{otherwise}.
 \end{align*}
 Interestingly, it suggests that if the voters are sufficiently diverse, \ie, 
 if $\Esp_{h\sim\post} \Risktrue(h) \le  \dis_{\Dcal}(\post)$, 
one should consider the C-bound as a surrogate to minimize.

\looseness=-1
Importantly, these approaches do not necessarily lead to tight guarantees for the majority vote, as they rely on surrogate quantities. 
Recently, \citet{zantedeschi2021learning} proposed a different perspective allowing PAC-Bayes to reason directly on the risk of majority votes.

\subsection{The PAC-Bayesian Stochastic Majority Vote} 
\label{sec:zantedeschi}
\looseness=-1
The idea of \citet{zantedeschi2021learning} is to shift the PAC-Bayesian randomization from the voter space $\Hcal$ to the space of majority votes itself.
Instead of defining a posterior distribution over individual voters $h \!\in\! \Hcal$, they consider a posterior over the weight vectors that define weighted majority votes, \ie, over the weights induced by $\post$.
Then, the stochasticity no longer comes from sampling a voter but from sampling a majority vote.
Concretely, let \mbox{$\Wcal\subseteq\Rbb^{|\!\Hcal\!|}$} denote the set of admissible weight vectors.
For any $\brho\in\Wcal$, a weighted majority vote of \Cref{eq:mv} is 
\begin{align}
\label{eq:mv-weight}
    \MV_{\brho}(x) =  \argmax_{y\in\Ycal}
    \left\{ \sum_{h\in\Hcal} \rho_h\,\I [h(x) = y] \right\}.
\end{align}
\looseness=-1
Then, they consider a \textit{hyper-prior} distribution $\hyperprior\in\Mcal^*(\Wcal)$ and a \textit{hyper-posterior} distribution $\hyperposterior\in\Mcal(\Wcal)$ defined on the weight space $\Wcal$.
Sampling $\brho\sim\hyperposterior$ amounts to sampling one deterministic majority vote~$\MV_{\brho}$.
Therefore, the stochastic majority vote is the randomized predictor obtained by drawing such a weight vector from the hyper-posterior before applying the corresponding majority vote.
Its true risk is given by
$\Esp_{\brho\sim\hyperposterior}\Risktrue(\MV_{\brho})$.
Under this framework, they derived the following PAC-Bayesian generalization bound.
\begin{theorem}[\small PAC-Bayesian generalization bound for the stochastic majority vote of \citet{zantedeschi2021learning}]\label{theo:zantedeschi}
For any distribution $\Dcal$, finite hypothesis set $\Hcal$, hyper-prior distribution $\hyperprior\in\Mcal^*(\Wcal)$ and $\delta\in(0,1]$, we have
\begin{align*}  
\Prob_{S\sim\Dcal^n}\!\left[
\forall \hyperposterior\in\Mcal(\Wcal), 
\Esp_{\brho \sim \hyperposterior} \Risktrue(\MV_{\brho})
 \leq 
\kl^{-1} \!\left(
\Esp_{\brho \sim \hyperposterior} \Riskemp(\MV_{\brho})
\bigg\|
\frac{\KL(\hyperposterior\|\hyperprior) + \ln \frac{2\sqrt{n}}{\delta}}{n}
\right)
\right]\geq 1\!-\!\delta,
\end{align*}
where $\kl^{-1}(q\|\epsilon) = \max\left\{ p \in [0,1] \ |\  \kl(q,p) \leq \epsilon \right\}$ is the inverse of the binary KL divergence.
\end{theorem}
\looseness=-1
\Cref{theo:zantedeschi} is appealing since it is expressed directly in terms of majority votes, without relying on a surrogate quantity as in \Cref{sec:detMV}. 
However, the empirical risk $\Esp_{\brho \sim \hyperposterior} \Riskemp(\MV_{\brho})$ is not straightforward to compute or optimize, since it involves an expectation over majority votes sampled from~$Q$.
To make the problem tractable, \citet{zantedeschi2021learning} consider the set $\Wcal\triangleq\{\brho\in[0,1]^{|\!\Hcal\!|}\;\big|\;\|\brho\|_1 = 1\}$ and model distributions over $\Wcal$ with Dirichlet distributions.
Specifically, a weight vector is sampled as  $\brho\sim\Dir(\balpha)$ with $\balpha\in\R_{>0}^{|\!\Hcal\!|}$.
In binary classification, this choice is particularly interesting since the expected $0$-$1$ loss admits a closed-form expression.
As a consequence, the empirical risk and the penalty can be combined into an explicit objective over the Dirichlet parameters.
The hyper-posterior $Q$ can then be learned by directly minimizing the resulting PAC-Bayesian bound, leading to a self-bounding learning algorithm for stochastic majority votes.

Nevertheless, the certified predictor remains stochastic, since the guarantee controls the average risk of majority votes drawn from $Q$.
This implies a drawback that has a consequence not only for computing the bound, but also at prediction time: for each data point, one has to sample a weight vector~$\brho\sim\hyperposterior$ and then use the induced majority vote $\MV_{\brho}$.
In the following, we show how to derandomize the stochastic majority vote by applying disintegrated PAC-Bayesian theory directly on the weight space $\Wcal$.

\section{Disintegrated stochastic majority vote}
\label{sec:theorems}
\label{sec:contribution}
\looseness=-1
We now present our contributions in this section. 
To obtain a certified deterministic majority vote, we leverage the recent advances in disintegrated PAC-Bayesian theory (\Cref{sec:disintegration}).
Since each weight vector $\brho\in\Wcal$ defines a deterministic majority vote~$\MV_{\brho}$, we apply the disintegration directly to the hyper-posterior over majority vote weights.
Consequently, if a learning algorithm outputs a hyper-posterior $\hyperposterior_S\in\Mcal(\Wcal)$, sampling $\brho\sim\hyperposterior_S$ amounts to selecting a single deterministic majority vote, for which disintegrated PAC-Bayesian theory provides high-probability generalization guarantees.
Importantly, this construction preserves the possibility of directly optimizing the resulting bounds in a self-bounding learning procedure (see \Cref{sec:algo}).
More precisely, our objective is to transform the guarantee of \Cref{theo:zantedeschi} for a stochastic classifier into a guarantee for a single majority vote.
Given a learning set $S$, the learner outputs a hyper-posterior~$Q_S$ over $\Wcal$, from which a single $\posterior$ weight vector is drawn.
We then seek a high-probability upper bound on $\Risktrue(\MV_{\brho})$. 

\looseness=-1
We consider two common settings in PAC-Bayes.
In \Cref{subsec:data-independent}, we assume the data-independent setting,  for which the set of voters $\Hcal$, the weight space $\Wcal$, and the hyper-prior $\hyperprior$ are fixed independently of the learning set.
This yields two disintegrated PAC-Bayesian guarantees: \Cref{thm:rivasplata_smv}, derived from \citet{rivasplata2020pac}, and \Cref{thm:viallard_smv}, derived from \citet{viallard2024general}.
Then, in \Cref{subsec:data-dependent}, we extend our analysis to data-dependent voters and priors through a sample-splitting strategy, leading to \Cref{thm:rivasplata_smv_DDBC,thm:viallard_smv_DDBC}.
This allows the voters and the hyper-priors to be learned from data while preserving PAC-Bayesian validity.

\subsection{Data-Independent Base Voters and Priors}\label{subsec:data-independent}
\looseness=-1
In the context of stochastic majority votes, the standard data-independent setting means that the set of voters $\Hcal$, from which the majority votes are built, is fixed before observing the data in $S$. 
Consequently, the weight space $\Wcal$ is fixed as well, and the hyper-prior distribution $\hyperprior \in \Mcal^*(\Wcal)$ is chosen independently of $S$.
The learning algorithm $A$ then uses the learning set $S \sim \Dcal^n$ to output a hyper-posterior $\hyperposterior_S \in \Mcal(\Wcal)$.

\looseness=-1
Technically, the bounds derivation consists of applying the disintegrated PAC-Bayesian theorems of \Cref{sec:disintegration} with the weight space $\Wcal$ ``playing the role'' of the hypothesis space. 
Therefore, the bounds hold for the drawing of a single weight vector $\posterior \in \Wcal$ that parametrizes the deterministic majority vote $\MV_{\posterior}$. 
Consequently, the empirical and true risks appearing in the bounds are $\Riskemp(\MV_{\posterior})$ and $\Risktrue(\MV_{\posterior})$.

\looseness=-1
First, in \Cref{thm:rivasplata_smv}, we leverage \Cref{thm:rivasplata} \citep{rivasplata2020pac}, in which  the penalty term takes the form of a pointwise log-density ratio $
\ln \frac{\hyperposterior_S(\posterior)}{\hyperprior(\posterior)}$ between the learned hyper-posterior and the hyper-prior, evaluated at the sampled weight vector~$\posterior$. 
This term compares how likely it is to sample $\posterior$ under the learned hyper-posterior and under the hyper-prior.
\begin{restatable}[\small Disintegrated bound for stochastic majority votes]{theorem}{rivasplataSMV}\label{thm:rivasplata_smv}
For any distribution $\Dcal$, finite hypothesis set~$\Hcal$, majority votes weight space $\Wcal\subseteq\Rbb^{|\!\Hcal\!|}$, hyper-prior $\hyperprior\in\Mcal^*(\Wcal)$, loss $\ell:\widehat{\Ycal}\times\Ycal\to[0,1]$, algorithm $A:(\Xcal{\times}\Ycal)^n \times \Mcal^*(\Wcal) \rightarrow \Mcal(\Wcal)$, and $\delta\in(0,1]$, we have
\begin{align*}
\Prob_{\substack{S\sim\Dcal^n \\ \brho \sim \hyperposterior_S}} \left[\kl(\Riskemp(\MV_{\brho})\|\Risktrue(\MV_{\brho}))\leq \frac{1}{n}\!\left(\ln \frac{\hyperposterior_S(\brho)}{\hyperprior(\brho)}+\ln\frac{2\sqrt{n}}{\delta}\right)\right] \geq 1\!-\!\delta, \ \text{with} \ \hyperposterior_S \triangleq A(S, \hyperprior).
\end{align*}
\end{restatable}
\looseness=-1
\begin{proof}
Theorem~\ref{thm:rivasplata} with ``$\Hcal \triangleq \Wcal$'' and $\displaystyle \varphi(\brho,S)
\triangleq n \, \kl\big(\Riskemp(\MV_{\brho})\,\|\,\Risktrue(\MV_{\brho})\big)$
leads to
\begin{align*}
\underset{\substack{S \sim \Dcal^n\\\brho \sim \hyperposterior_{S}}}{\mathbb{P}}
\left[ n \, \kl\big(\Riskemp(\MV_{\brho})\,\|\,\Risktrue(\MV_{\brho})\big) 
\leq \ 
\ln \left[\frac{\hyperposterior_{S}(\posterior)}{\hyperprior(\posterior)}\right]
+
\ln \left(\frac1\delta\, 
\underset{S^{\prime} \sim \Dcal^n}{\E}\,  \underset{\posterior^{\prime} \sim \hyperprior}{\E}\  e^{n \, \kl\big(\widehat{\Risk}_{S'}(\MV_{\brho'})\,\|\,\Risktrue(\MV_{\brho'})\big)} \right)
\right] \geq 1\!-\!\delta.
\end{align*}
For any $\brho \in \Wcal$, the risk $\Riskemp(\MV_{\brho})$ is the empirical mean of \iid random variables in $[0,1]$.
Thus, by applying an exponential moment inequality  for the empirical KL divergence of bounded \iid variables \citep{maurer2004note}, we have
$\Esp_{S' \sim \Dcal^n} \Esp_{\brho' \sim P}
e^{n \, \kl(\widehat{R}_{S'}(\MV_{\brho'})\,\|\,\Risktrue(\MV_{\brho'}))}
\le 2\sqrt{n}$.
Plugging into the previous expression leads to
$$
\underset{\substack{S \sim \Dcal^n\\\posterior \sim \hyperposterior_{S}}}{\mathbb{P}} \left[n\,\kl\left(\Riskemp(\MV_{\brho})\| \Risktrue(\MV_{\brho})\right)\leq \ln \frac{\hyperposterior_S(\brho)}{\hyperprior(\brho)}+\ln\frac{2\sqrt{n}}{\delta}\, \right] \geq 1-\delta.
$$
Dividing by $n$ concludes the proof.
\end{proof}
\looseness=-1
Second, we derive in~\Cref{thm:viallard_smv} a guarantee based on \Cref{thm:viallard} \citep{viallard2024general}.
Here, the penalty term is not evaluated using the sampled weight vector; instead, it depends on the Rényi divergence $\DivRen(\hyperposterior_S \| \hyperprior)$ between the whole distributions $\hyperposterior_S$ and  
$\hyperprior$.

\begin{restatable}[\small Rényi divergence-based disintegrated bound for stochastic majority votes]{theorem}{viallardSMV}\label{thm:viallard_smv}
For any distribution $\Dcal$, finite hypothesis set $\Hcal$, majority votes weight space $\Wcal\subseteq\Rbb^{|\!\Hcal\!|}$, hyper-prior $\hyperprior\in\Mcal^*(\Wcal)$, loss \mbox{$\ell:\widehat{\Ycal}\times\Ycal\to[0,1]$}, $\lambda>1$, algorithm $A:(\Xcal{\times}\Ycal)^n \times \Mcal^*(\Wcal) \rightarrow \Mcal(\Wcal)$, and $\delta\in(0,1]$, we have
\begin{align*}\Prob_{\substack{S\sim\Dcal^n\\\brho\sim \hyperposterior_S}}\left[
    \kl\Bigl(\Riskemp(\MV_{\brho}) \;\big\|\; \Risktrue(\MV_{\brho})\Bigr) \le \frac{1}{n} \left(
      \frac{2 \lambda\!-\!1}{\lambda\!-\!1}\ln \frac{2}{\delta} + \DivRen(\hyperposterior_S \| \hyperprior) + \ln(2 \sqrt{n}) \right)
\right] \ge 1\!-\!\delta,\ \mbox{with} \ \hyperposterior_S \triangleq A(S, \hyperprior).
\end{align*}
\end{restatable}
\begin{proof}
First, we apply \Cref{thm:viallard} with  ``$\Hcal \!\triangleq\! \Wcal$'' and $
\varphi(\brho,S) \!\triangleq\! e^{n \frac{\lambda-1}{\lambda}\kl\big(\Riskemp\left(\MV_{\brho}\right) \| \Risktrue\left(\MV_{\brho}\right)\big)}.$ 
For any 
$\brho'\! \in \!\Wcal$, the 
\textit{r.v.} $\Riskemp(\MV_{\brho'})$ is an empirical mean of \iid bounded variables in $[0,1]$. Hence, 
as in \Cref{thm:rivasplata_smv}, we have
\begin{align*}
\Esp_{S' \sim \Dcal^n}\Esp_{\brho' \sim P}
 \left(
e^{n\frac{\lambda-1}{\lambda}
\kl\left(
\widehat R_{S'}(\MV_{\brho'})
\middle\|
\Risktrue(\MV_{\brho'})
\right)}
\right)^{\frac{\lambda}{\lambda-1}}
\!= \!\Esp_{S' \sim \Dcal^n}\Esp_{\brho' \sim \hyperprior}
 e^{n \kl\left(\widehat R_{S'}\left(\MV_{\brho'}\right) \| \Risktrue\left(\MV_{\brho'}\right)\right)}
\le 2\sqrt{n}. 
\end{align*}
Combining this inequality with Theorem~\ref{thm:viallard} we have
\begin{align*}
\Prob_{\substack{S\sim\Dcal^n\\\brho\sim \hyperposterior_S}}\left[ n \kl\big(\Riskemp(\MV_{\brho})\,\|\,\Risktrue(\MV_{\brho})\big)
\le
\frac{2\lambda-1}{\lambda-1}\ln \frac{2}{\delta}
+ \DivRen(\hyperposterior_S \| \hyperprior)
+ \ln(2\sqrt{n})\right]\geq 1-\delta.
\end{align*}
Dividing by $n$ concludes the proof.
\end{proof}
\looseness=-1
As mentioned earlier, the difference between \Cref{thm:rivasplata_smv,thm:viallard_smv} lies in the penalty term.
Indeed, \Cref{thm:rivasplata_smv} involves a pointwise density ratio evaluated at the sampled weight vector, meaning its value depends only on the sampled $\post \sim \hyperposterior_S$.
In particular, it can be small when $\post$ has similar densities under the hyper-prior and the hyper-posterior, potentially leading to a tighter certificate.
However, from an algorithmic point of view, this can make it 
more difficult to optimize.
In contrast, \Cref{thm:viallard_smv} evaluates the penalty for the entire hyper-prior and hyper-posterior distributions, which generally leads to a more stable and tractable optimization objective, since it does not depend on a particular sampled weight vector.
In both cases, the result is a PAC-Bayesian bound on the risk of the deterministic majority vote $\MV_{\post}$ associated with a weight vector $\post \sim \hyperposterior_S$.

\subsection{Data-Dependent Base Voters and Priors }
\label{subsec:data-dependent}
\looseness=-1
The data-independent bounds of \Cref{subsec:data-independent} fall within the most classical PAC-Bayesian framework, which can be used when the learning task is simple enough that a fixed, data-independent family of voters can provide a sufficiently rich representation.
However, for more complex tasks (\eg, multiclass classification or with high-dimensional data), initializing voters that correctly cover the feature space while allowing for every possible label would require a large number of base voters.
This would increase the dimension of the hypothesis set, and thus the dimension of the associated weight space $\Wcal$, which may deteriorate the PAC-Bayesian bound and the ability to optimize the bound through the penalty term.
To overcome this drawback, we propose, following~\cite{mhammedi2019pac}, to consider data-dependent base voters using a cross-bounding certificate\footnote{
Note that there is a vast literature in PAC-Bayes that allows data-dependent base voters while also having valid generalization bounds \citep[\eg,][]{thiemann2017strongly,mhammedi2019pac,dziugaite2018data,zantedeschi2021learning}.}.
\looseness=-1
To do so, we split the learning set $S$ into two disjoint subsets $\Sone = \{(x_i, y_i)\}_{i=1}^\none$ and \mbox{$\Stwo  = \{(x_i, y_i)\}_{i=1}^\ntwo$} for $\none +\ntwo=n$.
The key idea is that each subset is used to learn a base voter set and a hyper-prior that will be evaluated on the other subset; for instance, $\Stwo$ is used to learn the base voters $\Hcalone$, to define the weight space $\Wcalone$, and to instantiate the hyper-prior $\hyperpriorone$.
This ensures that $\Hcalone$, $\Wcalone$ and $\hyperpriorone$ are independent of $\Sone$.
Symmetrically, $\Sone$ is used to learn $\Hcaltwo$, to define $\Wcaltwo$, and to instantiate~$\hyperpriortwo$, so that they are independent of $\Stwo$.
Thus, unlike the data-independent setting, both the base classifiers and their posterior weights can be learned from data.
Concretely, the algorithm~\textit{(i)} learns from $\Sone$ a hyper-posterior~$\hyperposterior_{\Sone }$, then samples a weight vector $\posteriorone$ to define the deterministic majority vote $\MV_{\posteriorone}$, and \textit{(ii)} learns from $\Stwo $ a hyper-posterior $\hyperposterior_{\Stwo }$, then  samples $\posteriortwo$ defining $\MV_{\posteriortwo}$.
The resulting bounds control a convex combination of the risks of the two sampled majority votes, weighted by a parameter~\mbox{$\tau\in[0,1]$}. 
We now extend the data-independent guarantees to data-dependent priors and base voters.
More precisely, \Cref{thm:rivasplata_smv_DDBC} extends \Cref{thm:rivasplata_smv}, while \Cref{thm:viallard_smv_DDBC} extends \Cref{thm:viallard_smv}.
\begin{restatable}[\small Data-dependent disintegrated bound for stochastic majority votes]{theorem}{rivasplataSmvDDBC}\label{thm:rivasplata_smv_DDBC}\looseness=-1
For any distribution $\Dcal$, sizes $(\none,\ntwo)$  such that $\none\!+\!\ntwo\!=\!n$, finite hypothesis sets $\Hcalone$ and $\Hcaltwo$, majority votes weight spaces $\Wcalone \!\subseteq\!\Rbb^{|\!\Hcalone\!|}$ and $\Wcaltwo\!\subseteq\!\Rbb^{|\!\Hcaltwo\!|}$, hyper-prior distributions $\hyperpriorone \in\Mcal^*(\Wcalone )$ and $\hyperpriortwo \in\Mcal^*(\Wcaltwo)$, loss $\ell:\widehat{\Ycal}{\times}\Ycal\!\to\![0,1]$, algorithms $\Aone:(\Xcal {\times} \Ycal)^\none {\times} \Mcal^*(\Wcalone )\!\to\! \Mcal(\Wcalone )$ and $\Atwo :(\Xcal {\times} \Ycal)^\ntwo {\times} \Mcal^*(\Wcaltwo)\!\to\! \Mcal(\Wcaltwo)$, and $\delta\!\in\! (0,1]$ and $\tau\!\in\![0,1]$, we have 
\begin{align*}
\Prob_{\substack{S\sim\Dcal^n \\ \posteriorone  \sim \hyperposterior_{\Sone } \\ \posteriortwo  \sim \hyperposterior_{\Stwo }}}\!\left[ \!
\begin{array}{r}
\operatorname{kl}\left( \tau \widehat{R}_{\Sone}\!(\MV_{\posteriorone })\! +\! (1\!-\!\tau)\, \widehat{R}_{\Stwo}\!(\MV_{\posteriortwo })
    \;\middle\|\;
    \tau  \Risktrue(\MV_{\posteriorone }) \!+\! (1\!-\!\tau) \Risktrue(\MV_{\posteriortwo })
\right)
\\[2mm]
\displaystyle
\leq \frac{\tau}{\none}\!\left[\ln \frac{\hyperposterior_{\Sone }(\posteriorone )}{\hyperpriorone(\posteriorone )}+\ln\frac{4\sqrt{\none}}{\delta}\right]+
\frac{1\!-\!\tau}{\ntwo}\!\left[\ln \frac{\hyperposterior_{\Stwo }(\posteriortwo )}{\hyperpriortwo (\posteriortwo )}+\ln\frac{4\sqrt{\ntwo}}{\delta}\right]
\end{array}\!\right]\geq 1-\delta,
\end{align*}
where $\hyperposterior_{\Sone } \triangleq \Aone(\Sone , \hyperpriorone)$ and $\hyperposterior_{\Stwo } \triangleq \Atwo(\Stwo , \hyperpriortwo)$.
\end{restatable}
\begin{restatable}[\small Data-dependent Rényi-based disintegrated bound for stochastic majority votes]{theorem}{viallardSmvDDBC}\label{thm:viallard_smv_DDBC}
Under the same assumptions as \Cref{thm:rivasplata_smv_DDBC} and with $\lambda>1$, we have
\begin{align*}
\Prob_{\substack{S\sim\Dcal^n \\ \posteriorone  \sim \hyperposterior_{\Sone } \\ \posteriortwo  \sim \hyperposterior_{\Stwo }}}\!\left[ \!
\begin{array}{l}
\operatorname{kl}\left(
    \tau \widehat{R}_{\Sone }(\MV_{\posteriorone }) \!+\! (1\!-\!\tau) \widehat{R}_{\Stwo }(\MV_{\posteriortwo })
    \;\middle\|\;
    \tau  \Risktrue(\MV_{\posteriorone }) \!+\! (1\!-\!\tau) \Risktrue(\MV_{\posteriortwo })
\right)\ \leq 
\\[2mm]\displaystyle
\frac{\tau}{\none}\Bigg[\frac{2 \lambda\!-\!1}{\lambda\!-\!1}\ln \frac{4}{\delta} \!+\! \DivRen(\hyperposterior_{\Sone } \| \hyperpriorone) \!+\! \ln(2 \sqrt{\none})\Bigg]  
\!+\! \frac{1\!-\!\tau}{\ntwo}\left[\frac{2 \lambda\!-\!1}{\lambda\!-\!1}\ln \frac{4}{\delta} + \DivRen(\hyperposterior_{\Stwo } \| \hyperpriortwo ) \!+\! \ln(2 \sqrt{\ntwo})\right]
\end{array}\hspace{-1.5mm}\right]\geq 1\!-\!\delta.
 \end{align*}
\end{restatable}
\begin{proof} The proofs of \Cref{thm:rivasplata_smv_DDBC,thm:viallard_smv_DDBC} are similar and follow  the proof scheme of \citet{zantedeschi2021learning}.
See, respectively, Appendices~
\ref{app:proof_riv_DDBC} and \ref{app:proof_viallard_DDBC}. 
\end{proof}
\looseness=-1
The interpretation of \Cref{thm:rivasplata_smv_DDBC,thm:viallard_smv_DDBC} is similar to that of the data-independent case, except that the guarantee is applied separately to the two complementary splits of $S$.
Each split is used to evaluate a majority vote whose voters and hyper-prior have been constructed from the other split.

This cross-dependence ensures that the hyper-prior remains independent of the data used to compute the empirical risk, thereby preserving PAC-Bayesian validity despite data-dependent voters.
The main advantage of this construction is that it extends our framework beyond fixed hypothesis sets.
In particular, the base voters can themselves be learned from data (\eg, using random forests or other multiclass predictors), after which our guarantees apply to the resulting learned voter sets.
Thus, \Cref{thm:rivasplata_smv_DDBC,thm:viallard_smv_DDBC} extend the disintegration principle to practical settings where both the voters and the hyper-posterior over their weights are learned from data.

In both the data-independent (\Cref{subsec:data-independent}) and data-dependent (\Cref{subsec:data-dependent}) settings, the guarantees apply to deterministic majority votes.
The former bounds the risk of a single sampled majority vote, whereas the latter bounds the risk of a convex combination of the risks of two sampled majority votes drawn from independent data splits.
Unlike the stochastic majority vote bound of \Cref{theo:zantedeschi}, neither result involves an expectation over majority votes: once the weight vectors have been sampled, the predictors are fixed and deterministic.
This provides the theoretical foundation of our self-bounding algorithm introduced below.

\section{Optimization Algorithm}\label{sec:algo}
\looseness=-1
A key characteristic of the bounds of \Cref{sec:contribution} is that they 
can serve as objective functions to minimize during training, leading 
to a self-bounding procedure where the generalization bound
is used as the optimization objective.
Note that self-bounding algorithms have recently regained interest in PAC-Bayes
\citep[\eg,][]{viallard2021self,zantedeschi2021learning,biggs2022margins,rivasplata2022pac,viallard2023pac,atbir2026pac}.

\looseness=-1
Although our disintegration analysis is not restricted to a specific family of hyper-prior and hyper-posterior distributions, its practical instantiation depends on the distribution chosen over the weight space~$\Wcal$.
The bounds apply to any pair of distributions for which the penalty term and empirical objective can be evaluated and optimized.
Following 
\citet{zantedeschi2021learning}, we use Dirichlet distributions and restrict the admissible weight vectors to the simplex $\Wcal = \{ \post \in [0,1]^{|\!\Hcal\!|} \,|\, \|\post\|_1 =1\}$ (for completeness, we provide the instantiation in~\Cref{app:dir specialization}).
Under this choice, our penalty terms admit closed-form expressions.

\looseness=-1
Since each theorem induces a specific learning objective (\Cref{eq:obj_rivasplata_dir,eq:obj_viallard_dir,eq:obj_rivasplata_DDBC_dir,eq:obj_viallard_DDBC_dir} of~\Cref{app:dirichlet_objectives}), we denote by $\mathcal B(\balpha;S,\hyperprior,\delta)$ the bound to minimize,  
where $\balpha$ parametrizes the Dirichlet hyper-posterior, $S$ is the learning set, $P$ the hyper-prior, and $\delta$ the confidence parameter. 
\Cref{alg} presents our generic learning procedure. 
\begin{algorithm}[H]
\caption{Self-bounding optimization of the Dirichlet hyper-posterior}
\label{alg}
\begin{algorithmic}[1]
\STATE \textbf{Input:} learning set $S$, voters $\Hcal$, hyper-prior $\hyperprior=\Dir(\bbeta)$ with $\bbeta\in\R_{>0}^{|\!\Hcal\!|}$, objective $\mathcal B$, number of epochs $T$, learning rate~$\eta$, confidence parameter $\delta$
\STATE \textbf{Output:} learned hyper-posterior $\hyperposterior_S=\Dir(\balpha)$ and deterministic majority vote $\MV_{\brho}$
\STATE Initialize the Dirichlet parameter $\balpha\in\R_{>0}^{|\!\Hcal\!|}$
\FOR{each epoch $t = 1,\ldots,T$}
\FOR{each mini-batch $\mathcal U \subseteq S$}
\STATE Draw a weight vector $\brho \sim \Dir(\balpha)$
\STATE Build the deterministic majority vote $\MV_{\brho}$
\STATE Compute the empirical risk $\widehat{\Risk}_{\mathcal{U}}(\MV_{\brho})$
\STATE Compute the self-bounding objective $\mathcal B(\balpha;\mathcal U,\hyperprior,\delta)$
\STATE Update $\balpha$ using gradient descent
\ENDFOR
\ENDFOR
\STATE Draw a final weight vector $\brho\sim\Dir(\balpha)$.
\STATE \textbf{Return:} $\hyperposterior_S=\Dir(\balpha)$ and $\MV_{\brho}$
\end{algorithmic}
\end{algorithm}
\looseness=-1
Once the hyper-posterior $\hyperposterior_S$ has been learned, the final sampled weight vector $\brho \sim \hyperposterior_S$ defines the deterministic classifier $\MV_{\brho}$ used at prediction time.
The same PAC-Bayesian guarantee that was minimized during training provides an upper bound on the true risk of this sampled deterministic majority vote.
An important computational advantage of this learning procedure, when compared to the stochastic majority vote, is that it does not require evaluating the expected empirical risk over all possible majority votes (as in \Cref{theo:zantedeschi} and \cite{zantedeschi2021learning}), but simply the empirical risk of the drawn majority vote.
Indeed, the expected empirical risk consists of an indefinite integral and thus must be approximated using either integration estimation or the Monte Carlo technique.
By contrast, computing the empirical risk of a drawn majority vote is straightforward and computationally cheap.

\section{Experiments}
\label{sec:expe}
\looseness=-1
\textbf{Compared methods.} Following the experimental protocol of \citet{zantedeschi2021learning}, we compare our \textbf{DIS}integrated methods with their \textbf{S}tochastic \textbf{M}ajority \textbf{V}ote ones. 
\mbox{\textsc{DIS-R}} denotes our pointwise-density-ratio-based bound of \Cref{thm:rivasplata_smv}, \textsc{DIS-V} the Rényi divergence-based bound of \Cref{thm:viallard_smv}, \textsc{SMV-Exact} the closed-form stochastic majority vote objective 
of \citet{zantedeschi2021learning}, and \textsc{SMV-MC} its Monte Carlo approximation.
Unlike \textsc{SMV-Exact} and \textsc{SMV-MC}, which certify the expected risk of a stochastic model, \textsc{DIS-R} and \textsc{DIS-V} certify a single deterministic model sampled from the learned hyper-posterior.
For completeness, we compare against three classical PAC-Bayesian majority vote methods:
the \textbf{F}irst-\textbf{O}rder method (\textsc{FO}) relies on the factor-two bound of \Cref{eq:factor-two}, while the \textbf{Bin}omial method (\textsc{Bin}) employs the binomial bound of~\Cref{eq:bin-bound}
and the \textbf{S}econd-\textbf{O}rder method (\textsc{SO}) uses the joint-error bound of \Cref{eq:second-order-bound}. 
The C-bound is not included as available optimization approaches are limited to binary classification and do not readily extend to large datasets.
\medskip 

\looseness=-1
\textbf{Datasets} (see \Cref{app:datasets} for details)\textbf{.}
We use binary and multiclass classification datasets from UCI Machine Learning Repository~\citep{dua2017uci}, LIBSVM repository, and Fashion-MNIST~\citep{xiao2017fashion}.

\medskip 

\looseness=-1
\textbf{Optimization.}
All methods are optimized using Adam~\citep{kingma2015adam}, with exponential moving-average coefficients $(0.9,0.999)$.
The learning rate is initialized to $0.1$ and is divided by $10$ after $2$ consecutive epochs without improvement.
Training is limited to $100$ epochs, with early stopping after $25$ epochs without improvement.
The data are split into $80\%$ training and a $20\%$ test set, using mini-batches of size $128$ for binary datasets and $1,024$ for multiclass datasets.
For \textsc{SMV-MC}, the expected empirical risk is estimated using $10$ Monte Carlo samples.
For all methods involving Monte Carlo sampling or disintegrated guarantees, the $0$-$1$ loss is approximated during training using a sigmoid surrogate with slope parameter $c\!=\!100$.
All bounds are computed with confidence parameter $\delta\!=\!0.05$, and the final bound is evaluated on the entire training set using the empirical $0$-$1$ loss.
Dirichlet hyper-priors are initialized with concentration parameter $\bbeta\!=\![0.5, 0.5, \dots]$ and the hyper-posterior parameters are initialized independently and uniformly in $[0.01,2]$.
For \textsc{FO}, \textsc{SO} and \textsc{Bin} 
priors are categorical distributions and posterior parameters are initialized uniformly in $[0.01,2]$ before normalization, and the $N$ parameter of \textsc{Bin} is fixed at $N=100$.
For \textsc{DIS-V}, the  Rényi divergence order is fixed to $\lambda\!=\!1.5$.
Each experiment is repeated $10$ times, and we report the mean and standard deviation of the bound value, empirical test risk, and training time.

\medskip 

\looseness=-1
\textbf{Base voters.}
For binary classification, we consider the data-independent setting of \Cref{subsec:data-independent}.
The voter set is composed of decision stumps.
For each input feature, $10$ decision thresholds are evenly distributed over the range of the feature space.
The resulting voter set, weight space, and hyper-prior are fixed independently of the learning set.
The hyper-posterior is then learned using the complete learning set. 
For multiclass classification, we consider the data-dependent setting of \Cref{subsec:data-dependent}.
The training sample is split into two equally sized subsets, and the value of the weighting parameter\footnote{The ratio $\nicefrac12$ is a standard choice in the PAC-Bayes literature \citep[\eg,][]{germain2009pac,letarte2019dichotomize,perezortiz2021tighter,zantedeschi2021learning,viallard2024general}.} $\tau$ is set to $\nicefrac12$.
Each subset is used to learn the voters and hyper-prior evaluated on the complementary subset.
For each split, the voters are the $M=100$ decision trees of a random forest trained on the opposite subset.
At each node, $\sqrt{d}$ features are randomly selected among the $d$ input features, the split is chosen by minimizing the Gini impurity, and the tree depth is left unconstrained.

\begin{figure}[t!]
    \centering
    \includegraphics[width=1.\linewidth]{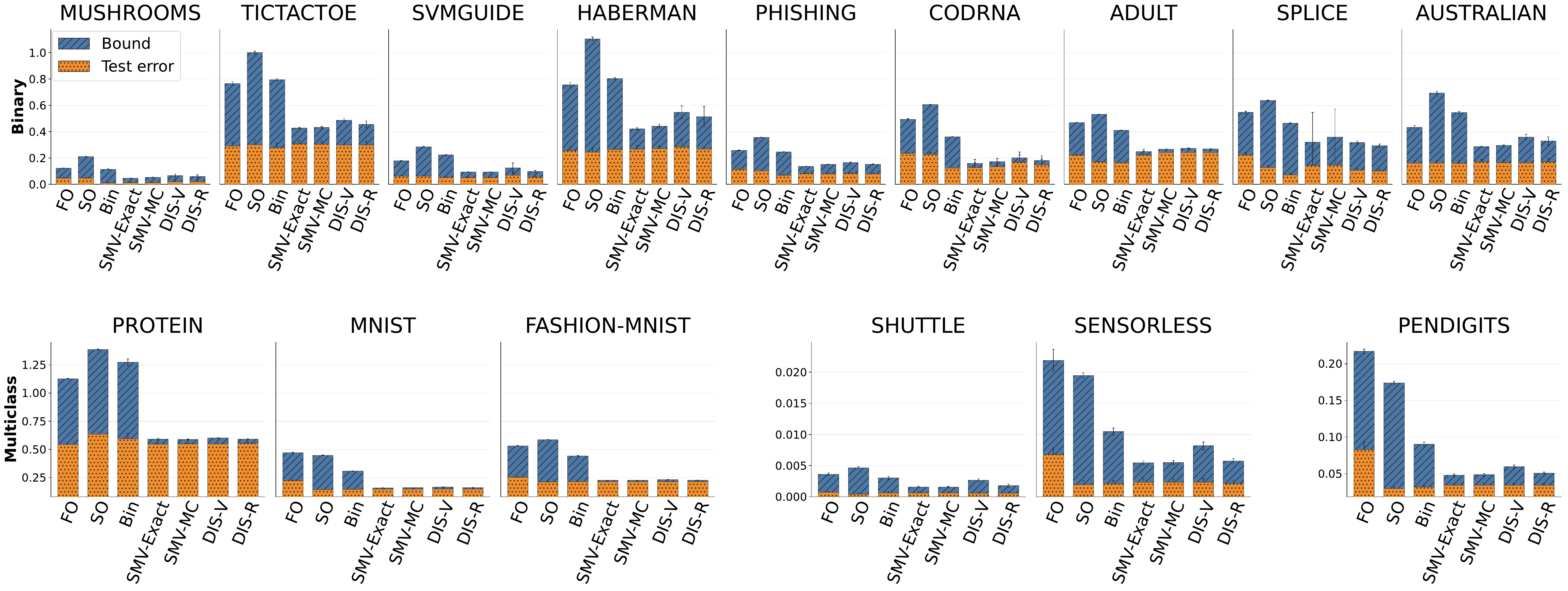}
    \caption{\looseness=-1 Test-error and PAC-Bayes bounds averaged over $10$  runs.
    Vertical lines show standard deviations.}
    \label{fig:best-bound-test-error}
\end{figure}

\medskip 

\looseness=-1
\textbf{Results---Test error and bound tightness.}
We report the results in~\Cref{fig:best-bound-test-error}.
Overall, all methods achieve comparable test errors.
In contrast, the behavior of the associated generalization bounds differs significantly between the methods that directly control the majority vote risk (\textsc{DIS-R}, \textsc{DIS-V}, \textsc{SMV-Exact}, \textsc{SMV-MC}) and the classical surrogate-based methods (\textsc{FO}, \textsc{SO}, \textsc{Bin}).
Across all datasets, \textsc{DIS-R}, \textsc{DIS-V}, \textsc{SMV-Exact},  and \textsc{SMV-MC} consistently yield substantially tighter bounds than the surrogate-based methods, whose bounds even become vacuous on the \textsc{Protein} dataset.\\
Comparing the stochastic majority vote with its disintegration, we observe that \textsc{SMV-Exact} and \textsc{SMV-MC} lead to the tightest bounds, with \textsc{DIS-R} outperforming them only on the \textsc{Splice} dataset.
This is expected, as they certify the expected risk of a stochastic predictor.
Nevertheless, our proposed disintegrated approaches (\textsc{DIS-R} and \textsc{DIS-V}) remain highly competitive while providing guarantees for a single deterministic majority vote, which constitutes a substantially stronger certification. 
As expected from the theory, \textsc{DIS-R} consistently produces tighter bounds than \textsc{DIS-V}.\\
Overall, these results show that certifying a single deterministic majority vote through disintegration instead of a stochastic majority vote only induces a very limited loss in bound tightness without degrading predictive performance.
The complete numerical results are reported in~\Cref{subsec:numerical_results}.

\begin{figure}[t!]
    \centering
    \includegraphics[width=1.\linewidth]{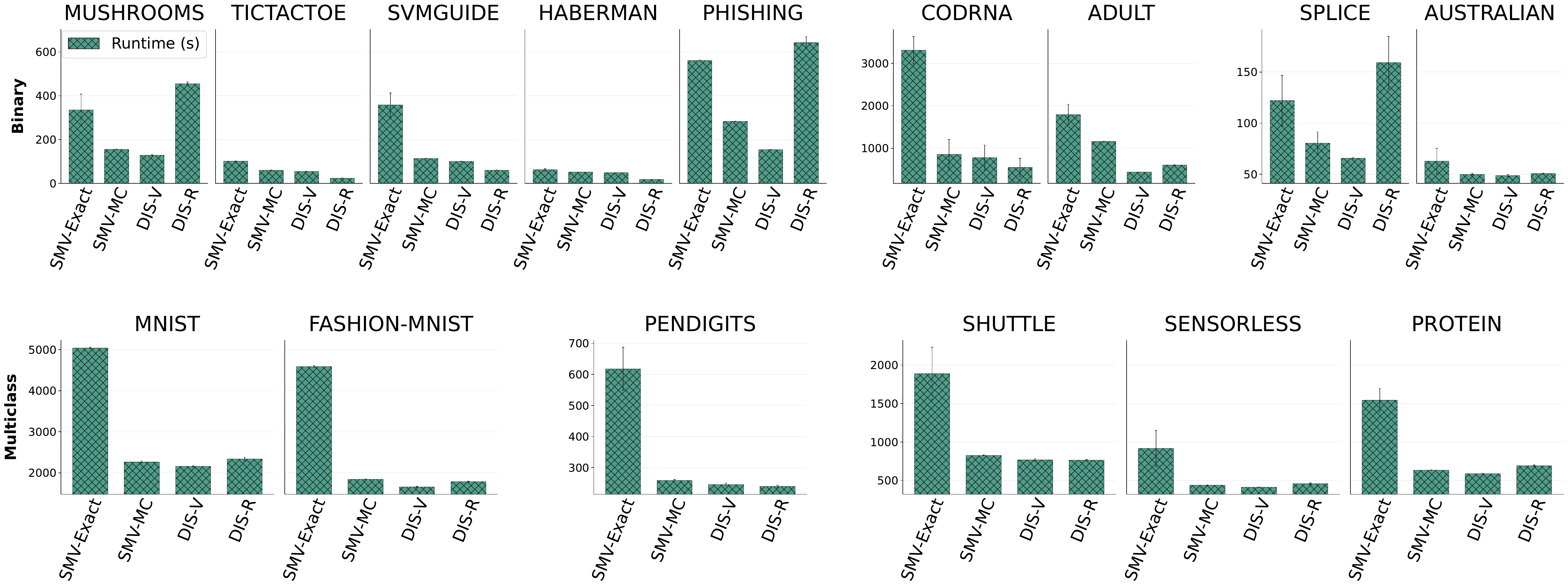}
    \caption{Training time across datasets, averaged over $10$ runs. Vertical lines show standard deviation.}
    \label{fig:best-time}
\end{figure}

\medskip 

\looseness=-1
\textbf{Results---Training times.}
We report in \Cref{fig:best-time} the training times of \textsc{SMV-Exact}, \textsc{SMV-MC}, \textsc{DIS-V} and \textsc{DIS-R}.
As expected, our disintegrated method \textsc{DIS-V} is faster on all datasets than the stochastic majority vote methods (\textsc{SMV-Exact} and \textsc{SMV-MC}).
Moreover, \textsc{DIS-R} is also faster on most datasets, with the exception of \textsc{Mushrooms}, \textsc{Phishing}, and \textsc{Splice}, where it is the slowest method.
The computational advantage of the disintegration is expected since our disintegrated objectives require optimizing only a single sampled majority vote at each mini-batch.
In contrast, \textsc{SMV-MC} approximates an expectation using multiple sampled majority votes, while \textsc{SMV-Exact} explicitly computes this expectation, resulting in a significantly higher computational cost.\\
Compared with \textsc{DIS-V}, the training time of \textsc{DIS-R} is more variable and is often slightly higher on the multiclass datasets.
This behavior is consistent with the additional optimization difficulty introduced by its sample-dependent pointwise density-ratio term, which can lead to a less stable optimization process and longer runtime on the \textsc{Mushrooms}, \textsc{Phishing}, and \textsc{Splice} datasets.\\
Overall, these results highlight a trade-off between optimization efficiency and certificate tightness.
While \textsc{DIS-V} provides the most computationally efficient objective, \textsc{DIS-R} tends to produce tighter certificates at the price of a more challenging optimization problem.

\section{Conclusion}
\label{sec:conclu}
\looseness=-1
In this paper, we introduce the first disintegrated PAC-Bayesian framework for stochastic majority votes.
Unlike classical PAC-Bayesian analyses of majority votes, which rely on surrogate quantities, our framework directly certifies the true risk of deterministic majority votes.
By applying recent disintegrated PAC-Bayesian theorems \citep{rivasplata2020pac,viallard2024general} to the space of majority vote weight vectors, we transform the stochastic majority vote guarantees of \citet{zantedeschi2021learning} into guarantees applied to a single sampled deterministic weighted majority vote, together with a self-bounding algorithm.
Empirically, our approach yields significantly tighter certificates than classical surrogate-based approaches while remaining competitive with stochastic majority votes and generally requiring less training time. 
A first natural extension is to investigate other families of distributions over predictor parameters beyond the Dirichlet family, which could further improve the flexibility 
of the learned majority votes while preserving tractable optimization.

\looseness=-1
More broadly and beyond majority votes, our work confirms that disintegrated PAC-Bayesian bounds naturally extend to higher levels of stochasticity, where the random object is no longer a base voter but an entire parameterization of the final predictor.
We believe that this perspective paves the way to new PAC-Bayesian self-bounding algorithms for 
modern machine learning models, where randomness can be introduced over latent representations, architectures, or other high-level model parameters rather than over individual predictors.

\subsection*{Acknowledgments.}
This work has been partly funded by public grants from the French
National Research Agency (ANR), namely the Famous project (ANR-23-CE23-0019), the DATeS project (ANR-25-CE23-6035).
Pascal Germain is supported by the NSERC Discovery grant RGPIN-2020-07223. 
Benjamin Leblanc is supported by a Mitacs Acceleration grant, in partnership with Intact Financial Corporation.  
Paul Viallard is partially funded through Inria with the associate team PACTOL and the exploratory action HYPE.

\bibliography{biblio}
\bibliographystyle{bibliostyle}

\appendix

\section{Proofs of the main results}

The two proofs rely on the cross-bounding argument described in
\Cref{subsec:data-dependent}. Recall that $\Hcalone$, $\Wcalone$, and
$\hyperpriorone$ are constructed from $\Stwo$ and are therefore
independent of $\Sone$. Symmetrically, $\Hcaltwo$, $\Wcaltwo$, and
$\hyperpriortwo$ are constructed from $\Sone$ and are independent of
$\Stwo$. We can thus condition on the subset used to construct the
voters and hyper-prior, apply the corresponding data-independent bound
to the other subset with confidence parameter $\delta/2$, and combine
the two resulting guarantees using a union bound and the joint
convexity of the binary KL divergence.

\subsection{Proof of~\Cref{thm:rivasplata_smv_DDBC}}
\label{app:proof_riv_DDBC}

\rivasplataSmvDDBC*

\begin{proof}
Recall that $\Hcalone$, $\Wcalone$, and $\hyperpriorone$ are constructed
from $\Stwo$ and are therefore independent of $\Sone$.
We can apply \Cref{thm:rivasplata_smv} to
$\Sone$ with confidence parameter $\delta/2$. We obtain
\begin{align*}
\Prob_{\substack{
\Sone\sim\Dcal^{\none}\\
\posteriorone\sim\hyperposterior_{\Sone}
}}
\Bigg[
&\kl\left(
\widehat R_{\Sone}(\MV_{\posteriorone})
\middle\|
\Risktrue(\MV_{\posteriorone})
\right) \leq
\frac{1}{\none}
\left[
\ln
\frac{
\hyperposterior_{\Sone}(\posteriorone)
}{
\hyperpriorone(\posteriorone)
}
+
\ln\frac{4\sqrt{\none}}{\delta}
\right]
\Bigg]
\geq 1-\frac{\delta}{2}.
\end{align*}

Adjoining the
draw $\posteriortwo\sim\hyperposterior_{\Stwo}$ gives
\begin{align}
\Prob_{\substack{
S\sim\Dcal^n\\
\posteriorone\sim\hyperposterior_{\Sone}\\
\posteriortwo\sim\hyperposterior_{\Stwo}
}}
\Bigg[
&\kl\left(
\widehat R_{\Sone}(\MV_{\posteriorone})
\middle\|
\Risktrue(\MV_{\posteriorone})
\right) \leq
\frac{1}{\none}
\left[
\ln
\frac{
\hyperposterior_{\Sone}(\posteriorone)
}{
\hyperpriorone(\posteriorone)
}
+
\ln\frac{4\sqrt{\none}}{\delta}
\right]
\Bigg]
\geq 1-\frac{\delta}{2}.
\label{eq:rivasplata-first-split}
\end{align}

Symmetrically, $\Hcaltwo$, $\Wcaltwo$, and $\hyperpriortwo$ are
constructed from $\Sone$ and are therefore independent of $\Stwo$, applying \Cref{thm:rivasplata_smv} to
$\Stwo$ with confidence parameter $\delta/2$ gives
\begin{align*}
\Prob_{\substack{
\Stwo\sim\Dcal^{\ntwo}\\
\posteriortwo\sim\hyperposterior_{\Stwo}
}}
\Bigg[
&\kl\left(
\widehat R_{\Stwo}(\MV_{\posteriortwo})
\middle\|
\Risktrue(\MV_{\posteriortwo})
\right)\leq
\frac{1}{\ntwo}
\left[
\ln
\frac{
\hyperposterior_{\Stwo}(\posteriortwo)
}{
\hyperpriortwo(\posteriortwo)
}
+
\ln\frac{4\sqrt{\ntwo}}{\delta}
\right]
\Bigg]
\geq 1-\frac{\delta}{2}.
\end{align*}

Adjoining the draw
$\posteriorone\sim\hyperposterior_{\Sone}$ yields
\begin{align}
\Prob_{\substack{
S\sim\Dcal^n\\
\posteriorone\sim\hyperposterior_{\Sone}\\
\posteriortwo\sim\hyperposterior_{\Stwo}
}}
\Bigg[
&\kl\left(
\widehat R_{\Stwo}(\MV_{\posteriortwo})
\middle\|
\Risktrue(\MV_{\posteriortwo})
\right)
\leq
\frac{1}{\ntwo}
\left[
\ln
\frac{
\hyperposterior_{\Stwo}(\posteriortwo)
}{
\hyperpriortwo(\posteriortwo)
}
+
\ln\frac{4\sqrt{\ntwo}}{\delta}
\right]
\Bigg]
\geq 1-\frac{\delta}{2}.
\label{eq:rivasplata-second-split}
\end{align}

By a union bound,
\Cref{eq:rivasplata-first-split,eq:rivasplata-second-split} hold
simultaneously with probability at least $1-\delta$. Therefore,

\begin{align*}
\Prob_{\substack{
S\sim\Dcal^n\\
\posteriorone\sim\hyperposterior_{\Sone}\\
\posteriortwo\sim\hyperposterior_{\Stwo}
}}
\!\left[
\begin{aligned}
\tau
\kl &\left(
\widehat{R}_{\Sone}(\MV_{\posteriorone})
\middle\| 
\Risktrue(\MV_{\posteriorone})
\right)
+
(1-\tau)
\kl\left(
\widehat{R}_{\Stwo}(\MV_{\posteriortwo})
\middle\|
\Risktrue(\MV_{\posteriortwo})
\right)
\\[1mm]
&\leq
\frac{\tau}{\none}
\left[
\ln
\frac{
\hyperposterior_{\Sone}(\posteriorone)
}{
\hyperpriorone(\posteriorone)
}
+
\ln\frac{4\sqrt{\none}}{\delta}
\right]\quad+
\frac{1-\tau}{\ntwo}
\left[
\ln
\frac{
\hyperposterior_{\Stwo}(\posteriortwo)
}{
\hyperpriortwo(\posteriortwo)
}
+
\ln\frac{4\sqrt{\ntwo}}{\delta}
\right]
\end{aligned}
\right]
\geq 1-\delta.
\end{align*}

Finally, by the convexity of the KL divergence,
\begin{align*}
&\tau
\kl\left(
\widehat R_{\Sone}(\MV_{\posteriorone})
\middle\|
\Risktrue(\MV_{\posteriorone})
\right)
+
(1-\tau)
\kl\left(
\widehat R_{\Stwo}(\MV_{\posteriortwo})
\middle\|
\Risktrue(\MV_{\posteriortwo})
\right)
\\
&\geq
\kl\left(
\tau\widehat R_{\Sone}(\MV_{\posteriorone})
+
(1-\tau)\widehat R_{\Stwo}(\MV_{\posteriortwo})
\middle\|
\tau\Risktrue(\MV_{\posteriorone})
+
(1-\tau)\Risktrue(\MV_{\posteriortwo})
\right).
\end{align*}
Substituting this inequality into the preceding probability statement
gives the claimed result.
\end{proof}

\subsection{Proof of~\Cref{thm:viallard_smv_DDBC}}
\label{app:proof_viallard_DDBC}

\viallardSmvDDBC*

\begin{proof}
Recall that $\Hcalone$, $\Wcalone$, and $\hyperpriorone$ are constructed
from $\Stwo$ and are therefore independent of $\Sone$.
We can apply \Cref{thm:viallard_smv} to
$\Sone$ with confidence parameter $\delta/2$. We obtain
\begin{equation*}
\Prob_{\substack{
\Sone\sim\Dcal^{\none}\\
\posteriorone\sim\hyperposterior_{\Sone}
}}
\!\left[
\kl\left(
\widehat R_{\Sone}(\MV_{\posteriorone})
\middle\|
\Risktrue(\MV_{\posteriorone})
\right) \leq
\frac{1}{\none}
\left[
\frac{2\lambda-1}{\lambda-1}
\ln\frac{4}{\delta}
+
\DivRen\left(
\hyperposterior_{\Sone}
\middle\|
\hyperpriorone
\right)
+
\ln(2\sqrt{\none})
\right]
\right] \geq 1-\frac{\delta}{2}.
\end{equation*}

Adjoining the
draw $\posteriortwo\sim\hyperposterior_{\Stwo}$ gives
\begin{equation}
\Prob_{\substack{
S\sim\Dcal^n\\
\posteriorone\sim\hyperposterior_{\Sone}\\
\posteriortwo\sim\hyperposterior_{\Stwo}
}}
\!\left[
\begin{aligned}
\kl\left(
\widehat R_{\Sone}(\MV_{\posteriorone})
\middle\|
\Risktrue(\MV_{\posteriorone})
\right)\leq
\frac{1}{\none}
\left[
\frac{2\lambda-1}{\lambda-1}
\ln\frac{4}{\delta}
+
\DivRen\left(
\hyperposterior_{\Sone}
\middle\|
\hyperpriorone
\right)
+
\ln(2\sqrt{\none})
\right]
\end{aligned}
\right]
\geq 1-\frac{\delta}{2}.
\label{eq:viallard-first-split}
\end{equation}

Symmetrically, $\Hcaltwo$, $\Wcaltwo$, and $\hyperpriortwo$ are
constructed from $\Sone$ and are therefore independent of $\Stwo$, applying \Cref{thm:viallard_smv} to $\Stwo$
with confidence parameter $\delta/2$ gives
\begin{equation*}
\Prob_{\substack{
\Stwo\sim\Dcal^{\ntwo}\\
\posteriortwo\sim\hyperposterior_{\Stwo}
}}
\!\left[
\kl\left(
\widehat R_{\Stwo}(\MV_{\posteriortwo})
\middle\|
\Risktrue(\MV_{\posteriortwo})
\right) \leq
\frac{1}{\ntwo}
\left[
\frac{2\lambda-1}{\lambda-1}
\ln\frac{4}{\delta}
+
\DivRen\left(
\hyperposterior_{\Stwo}
\middle\|
\hyperpriortwo
\right)
+
\ln(2\sqrt{\ntwo})
\right]
\right]
\geq 1-\frac{\delta}{2}.
\end{equation*}

Adjoining the draw
$\posteriorone\sim\hyperposterior_{\Sone}$ yields
\begin{equation}
\Prob_{\substack{
S\sim\Dcal^n\\
\posteriorone\sim\hyperposterior_{\Sone}\\
\posteriortwo\sim\hyperposterior_{\Stwo}
}}
\!\left[
\begin{aligned}
\kl\left(
\widehat R_{\Stwo}(\MV_{\posteriortwo})
\middle\|
\Risktrue(\MV_{\posteriortwo})
\right)
\leq
\frac{1}{\ntwo}
\left[
\frac{2\lambda-1}{\lambda-1}
\ln\frac{4}{\delta}
+
\DivRen\left(
\hyperposterior_{\Stwo}
\middle\|
\hyperpriortwo
\right)
+
\ln(2\sqrt{\ntwo})
\right]
\end{aligned}
\right]
\geq 1-\frac{\delta}{2}.
\label{eq:viallard-second-split}
\end{equation}

By a union bound,
\Cref{eq:viallard-first-split,eq:viallard-second-split} hold
simultaneously with probability at least $1-\delta$. Therefore,
\begin{equation*}
\Prob_{\substack{
S\sim\Dcal^n\\
\posteriorone\sim\hyperposterior_{\Sone}\\
\posteriortwo\sim\hyperposterior_{\Stwo}
}}
\!\left[
\begin{aligned}
&\tau
\kl\left(
\widehat R_{\Sone}(\MV_{\posteriorone})
\middle\|
\Risktrue(\MV_{\posteriorone})
\right) +
(1-\tau)
\kl\left(
\widehat R_{\Stwo}(\MV_{\posteriortwo})
\middle\|
\Risktrue(\MV_{\posteriortwo})
\right)
\\[1mm]
&\leq
\frac{\tau}{\none}
\left[
\frac{2\lambda-1}{\lambda-1}
\ln\frac{4}{\delta}
+
\DivRen\left(
\hyperposterior_{\Sone}
\middle\|
\hyperpriorone
\right)
+
\ln(2\sqrt{\none})
\right]
\\
&\quad+
\frac{1-\tau}{\ntwo}
\left[
\frac{2\lambda-1}{\lambda-1}
\ln\frac{4}{\delta}
+
\DivRen\left(
\hyperposterior_{\Stwo}
\middle\|
\hyperpriortwo
\right)
+
\ln(2\sqrt{\ntwo})
\right]
\end{aligned}
\right]
\geq 1-\delta.
\end{equation*}

Finally, by the convexity of the KL divergence,
\begin{align*}
&\tau
\kl\left(
\widehat R_{\Sone}(\MV_{\posteriorone})
\middle\|
\Risktrue(\MV_{\posteriorone})
\right)
+
(1-\tau)
\kl\left(
\widehat R_{\Stwo}(\MV_{\posteriortwo})
\middle\|
\Risktrue(\MV_{\posteriortwo})
\right)
\\
&\geq
\kl\left(
\tau\widehat R_{\Sone}(\MV_{\posteriorone})
+
(1-\tau)\widehat R_{\Stwo}(\MV_{\posteriortwo})
\middle\|
\tau\Risktrue(\MV_{\posteriorone})
+
(1-\tau)\Risktrue(\MV_{\posteriortwo})
\right).
\end{align*}
Substituting this inequality into the preceding probability statement
gives the claimed result.
\end{proof}
\section{Dirichlet Specializations and Optimization Objectives}
\label{app:dir specialization}
\label{app:dirichlet_objectives}

We specialize our bounds to Dirichlet hyper-priors and
hyper-posteriors, and derive the corresponding self-bounding
optimization objectives for both the data-independent and
data-dependent settings.

\subsection{Preliminaries on Dirichlet Distributions}
\label{app:dirichlet_preliminaries}

Let $\Wcal=\{\brho\in[0,1]^{|\!\Hcal\!|}:\|\brho\|_1=1\}$. For any
$\balpha\in\Rbb_{>0}^{|\!\Hcal\!|}$, the density of $\Dir(\balpha)$ is
\[
\Dir(\brho\mid\balpha)
=
\frac{1}{B(\balpha)}
\prod_{j=1}^{|\!\Hcal\!|}\rho_j^{\alpha_j-1},
\qquad
B(\balpha)
=
\frac{\prod_{j=1}^{|\!\Hcal\!|}\Gamma(\alpha_j)}
{\Gamma\left(\sum_{j=1}^{|\!\Hcal\!|}\alpha_j\right)}.
\]

Let $\hyperprior=\Dir(\bbeta)$ and
$\hyperposterior_S=\Dir(\balpha)$. Their pointwise log-density ratio is
\begin{equation}
\ln\frac{\hyperposterior_S(\brho)}{\hyperprior(\brho)}
=
\ln\frac{B(\bbeta)}{B(\balpha)}
+
\sum_{j=1}^{|\!\Hcal\!|}(\alpha_j-\beta_j)\ln(\rho_j).
\label{eq:KL_dis_Dir}
\end{equation}

We refer to \cite{gil2013renyi} for the expression of the Rényi divergence for the Dirichlet distribution. With $\prior= \Dir (\bbeta)$ the prior distribution and $\post_S= \Dir (\balpha)$ we have that for any $\lambda>1$
   \begin{equation}\label{eq:Rényi_dir}
       \DivRen(\post_S \| \prior) = \ln\frac{B(\bbeta)}{B(\balpha)} + \frac{1}{\lambda-1}\ln\frac{B(\lambda\balpha+(1-\lambda)\bbeta)}{B(\balpha)}
   \end{equation}
This expression is finite provided that
\[
\lambda\alpha_j+(1-\lambda)\beta_j>0,
\qquad j=1,\ldots,|\!\Hcal\!|,
\]
or, equivalently,
\[
\alpha_j>\frac{\lambda-1}{\lambda}\beta_j,
\qquad j=1,\ldots,|\!\Hcal\!|.
\]

Finally, we define
\[
\kl^{-1}(q\|\epsilon)
\triangleq
\max\left\{p\in[q,1]:\kl(q\|p)\leq\epsilon\right\}.
\]
Thus, $\kl(q\|p)\leq\epsilon$ implies
$p\leq\kl^{-1}(q\|\epsilon)$.

\subsection{Data-Independent Setting}
\label{app:dirichlet_data_independent}

\subsubsection{Pointwise Density-Ratio Bound}

\begin{corollary}[
\Cref{thm:rivasplata_smv} with Dirichlet distributions
]\label{cor:rivasplata_dir}
For any distribution $\Dcal$, finite hypothesis set $\Hcal$,
majority-vote weight space
$\Wcal=\{\brho\in[0,1]^{|\!\Hcal\!|}:\|\brho\|_1=1\}$,
hyper-prior distribution $\hyperprior=\Dir(\bbeta)$,
hyper-posterior distribution
$\hyperposterior_S=\Dir(\balpha)$,
loss $\ell:\widehat{\Ycal}\times\Ycal\to[0,1]$,
algorithm
$A:(\Xcal{\times}\Ycal)^n\times\Mcal^*(\Wcal)
\rightarrow\Mcal(\Wcal)$,
and $\delta\in(0,1]$,
with $\balpha,\bbeta\in\Rbb_{>0}^{|\!\Hcal\!|}$, we have
\begin{equation*}
\Prob_{\substack{S\sim\Dcal^n\\\brho\sim\Dir(\balpha)}}
\left[
\kl\left(\Riskemp(\MV_{\brho})\middle\|\Risktrue(\MV_{\brho})\right)
\leq
\frac{1}{n}
\left(
\ln\frac{B(\bbeta)}{B(\balpha)}
+
\sum_{j=1}^{|\!\Hcal\!|}(\alpha_j-\beta_j)\ln(\rho_j)
+
\ln\frac{2\sqrt n}{\delta}
\right)
\right]
\geq 1-\delta.
\end{equation*}
where $\hyperposterior_{S } \triangleq \A(S , \hyperprior)$.
\end{corollary}

\begin{proof}
Recall that the Dirichlet density is defined as
\[
\Dir(\brho \mid \balpha) = \frac{1}{B(\balpha)} \prod_{j=1}^{|\!\Hcal\!|} \rho_j^{\alpha_j - 1}.
\]

Then, for any $\brho \sim \post_S$ we obtain the following decomposition:
\begin{equation}\label{eq:KL_dis_Dir_pr}
\begin{aligned}
\ln \frac{\post_S(\brho)}{\prior(\brho)} 
&= \ln \post_S(\brho) - \ln \prior(\brho) \\
&= \Big[-\ln B(\balpha) + \sum_{j=1}^{|\!\Hcal\!|} (\alpha_j - 1) \ln \rho_j \Big] 
   - \Big[-\ln B(\bbeta) + \sum_{j=1}^{|\!\Hcal\!|} (\beta_j - 1) \ln \rho_j \Big] \\
&= \sum_{j=1}^{|\!\Hcal\!|} (\alpha_j - \beta_j) \ln \rho_j + \big( \ln B(\bbeta) - \ln B(\balpha) \big).
\end{aligned}
\end{equation}
Substituting \Cref{eq:KL_dis_Dir_pr} into~\Cref{thm:rivasplata_smv} yields the stated result.
\end{proof}

The corresponding self-bounding optimization objective is
\begin{equation}
\kl^{-1}\left(
\Riskemp(\MV_{\brho})
\;\middle\|\;
\frac{1}{n}
\left[
\ln\frac{B(\bbeta)}{B(\balpha)}
+
\sum_{j=1}^{|\!\Hcal\!|}(\alpha_j-\beta_j)\ln(\rho_j)
+
\ln\frac{2\sqrt n}{\delta}
\right]
\right).
\label{eq:obj_rivasplata_dir}
\end{equation}

\subsubsection{Rényi Divergence-Based Bound}

\begin{corollary}[
\Cref{thm:viallard_smv} with Dirichlet distributions
]
\label{cor:viallard_smv_dirichlet}
For any distribution $\Dcal$, finite hypothesis set $\Hcal$,
majority-vote weight space
$\Wcal=\{\brho\in[0,1]^{|\!\Hcal\!|}:\|\brho\|_1=1\}$,
hyper-prior distribution $\hyperprior=\Dir(\bbeta)$,
hyper-posterior distribution
$\hyperposterior_S=\Dir(\balpha)$,
loss $\ell:\widehat{\Ycal}\times\Ycal\to[0,1]$,
algorithm
$A:(\Xcal{\times}\Ycal)^n\times\Mcal^*(\Wcal)
\rightarrow\Mcal(\Wcal)$,
$\lambda>1$, and $\delta\in(0,1]$,
with $\balpha,\bbeta\in\Rbb_{>0}^{|\!\Hcal\!|}$, we have
\begin{equation*}
\Prob_{\substack{S\sim\Dcal^n\\\brho\sim\Dir(\balpha)}}
\left[ 
\begin{aligned}
\kl &\left(\Riskemp(\MV_{\brho})\middle\|\Risktrue(\MV_{\brho})\right)
\leq \\
& \frac{1}{n}
\left(
\frac{2\lambda-1}{\lambda-1}\ln\frac{2}{\delta}
+
\ln\frac{B(\bbeta)}{B(\balpha)} 
+   \frac{1}{\lambda-1}
\ln\frac{B\left(\lambda\balpha+(1-\lambda)\bbeta\right)}
{B(\balpha)}
+
\ln(2\sqrt n)
\right)
\end{aligned} \right]
\geq 1-\delta.
\end{equation*}
where $\hyperposterior_{S } \triangleq \A(S , \hyperprior)$.
\end{corollary}

\begin{proof}
Substituting the Rényi divergence of \Cref{eq:Rényi_dir} into
\Cref{thm:viallard_smv} gives the result.
\end{proof}

The corresponding self-bounding optimization objective is
\begin{equation}
\kl^{-1}\left(
\Riskemp(\MV_{\brho})
\;\middle\|\;
\frac{1}{n}
\left[
\frac{2\lambda-1}{\lambda-1}\ln\frac{2}{\delta}
+
\ln\frac{B(\bbeta)}{B(\balpha)}
+
\frac{1}{\lambda-1}
\ln\frac{B\left(\lambda\balpha+(1-\lambda)\bbeta\right)}
{B(\balpha)}
+
\ln(2\sqrt n)
\right]
\right).
\label{eq:obj_viallard_dir}
\end{equation}

\subsection{Data-Dependent Setting}
\label{app:dirichlet_data_dependent}

We consider $\none=\ntwo=n/2$ and $\tau=1/2$, assuming that $n$ is
even. We instantiate the hyper-priors and hyper-posteriors as
\[
\hyperpriorone=\Dir(\bbeta^{(1)}),
\qquad
\hyperposterior_{\Sone}=\Dir(\balpha^{(1)}),
\qquad
\hyperpriortwo=\Dir(\bbeta^{(2)}),
\qquad
\hyperposterior_{\Stwo}=\Dir(\balpha^{(2)}),
\]
where
\[
\balpha^{(1)},\bbeta^{(1)}
\in\Rbb_{>0}^{|\!\Hcalone\!|},
\qquad
\balpha^{(2)},\bbeta^{(2)}
\in\Rbb_{>0}^{|\!\Hcaltwo\!|}.
\]

\subsubsection{Pointwise Density-Ratio Bound}

\begin{corollary}[
\Cref{thm:rivasplata_smv_DDBC} with Dirichlet distributions
]
\label{cor:rivasplata_smv_DDBC_dirichlet}
For any distribution $\Dcal$, finite hypothesis sets $\Hcalone$ and $\Hcaltwo$, majority votes weight spaces $\Wcalone = \{\brho\in[0,1]^{|\!\Hcalone\!|}:\|\brho\|_1=1\}$, $\Wcaltwo = \{\brho\in[0,1]^{|\!\Hcaltwo\!|}:\|\brho\|_1=1\}$, hyper-prior distributions $\hyperpriorone=\Dir(\bbeta^{(1)})$, $\hyperpriortwo=\Dir(\bbeta^{(2)})$, hyper-posterior $\hyperposterior_{\Sone}=\Dir(\balpha^{(1)})$,  $\hyperposterior_{\Stwo}=\Dir(\balpha^{(2)})$, loss $\ell:\widehat{\Ycal}\times\Ycal\to[0,1]$, algorithms $\Aone:(\Xcal{\times}\Ycal)^\none \times \Mcal^*(\Wcalone )\rightarrow \Mcal(\Wcalone )$ and $\Atwo :(\Xcal{\times}\Ycal)^\ntwo \times \Mcal^*(\Wcaltwo)\rightarrow \Mcal(\Wcaltwo)$, and $\delta\in (0,1]$, with $\balpha^{(1)},\bbeta^{(1)}\in\Rbb_{>0}^{|\!\Hcalone\!|}$, $\balpha^{(2)},\bbeta^{(2)} \in\Rbb_{>0}^{|\!\Hcaltwo\!|}$, $\none=\ntwo=n/2$ and $\tau=1/2$, we have 

\begin{equation*}
\Prob_{\substack{
S\sim\Dcal^n\\
\posteriorone\sim\Dir(\balpha^{(1)})\\
\posteriortwo\sim\Dir(\balpha^{(2)})
}}
\left[
\begin{aligned}
\kl\left(
\frac12\widehat R_{\Sone}(\MV_{\posteriorone})
+
\frac12\widehat R_{\Stwo}(\MV_{\posteriortwo})
\middle\| 
\frac12\Risktrue(\MV_{\posteriorone})
+
\frac12\Risktrue(\MV_{\posteriortwo})
\right)&
\\
\leq
\frac{1}{n}
\Bigg[
\ln\frac{B(\bbeta^{(1)})B(\bbeta^{(2)})}
{B(\balpha^{(1)})B(\balpha^{(2)})}
+
\sum_{j=1}^{|\!\Hcalone\!|}
\left(\alpha_j^{(1)}-\beta_j^{(1)}\right)
\ln(\rho_1)_j +& \\
\sum_{j=1}^{|\!\Hcaltwo\!|}
\left(\alpha_j^{(2)}-\beta_j^{(2)}\right)
\ln(\rho_2)_j
+&
\ln\frac{8n}{\delta^2}
\Bigg]
\end{aligned}
\right]
\geq 1-\delta.
\end{equation*}
where $\hyperposterior_{\Sone } \triangleq \Aone(\Sone , \hyperpriorone)$ and $\hyperposterior_{\Stwo } \triangleq \Atwo(\Stwo , \hyperpriortwo)$.
\end{corollary}

\begin{proof}
Apply \Cref{eq:KL_dis_Dir} to the pairs
$(\hyperposterior_{\Sone},\hyperpriorone)$ and
$(\hyperposterior_{\Stwo},\hyperpriortwo)$ in
\Cref{thm:rivasplata_smv_DDBC}. Since
\[
\frac{\tau}{\none}
=
\frac{1-\tau}{\ntwo}
=
\frac{1}{n}\qquad \text{and}
\qquad
2\ln\frac{4\sqrt{n/2}}{\delta}
=
\ln\frac{8n}{\delta^2},
\]
the result follows.
\end{proof}

The corresponding self-bounding optimization objective is
\begin{equation}
\begin{aligned}
\kl^{-1}\Bigg(
&\frac12\widehat R_{\Sone}(\MV_{\posteriorone})
+
\frac12\widehat R_{\Stwo}(\MV_{\posteriortwo})
\;\Bigg\|\;
\frac{1}{n}
\Bigg[
\ln\frac{B(\bbeta^{(1)})B(\bbeta^{(2)})}
{B(\balpha^{(1)})B(\balpha^{(2)})}
\\
&\qquad+
\sum_{j=1}^{|\!\Hcalone\!|}
\left(\alpha_j^{(1)}-\beta_j^{(1)}\right)
\ln(\rho)_j
+
\sum_{j=1}^{|\!\Hcaltwo\!|}
\left(\alpha_j^{(2)}-\beta_j^{(2)}\right)
\ln(\rho)_j
+
\ln\frac{8n}{\delta^2}
\Bigg]
\Bigg).
\end{aligned}
\label{eq:obj_rivasplata_DDBC_dir}
\end{equation}

\subsubsection{Rényi Divergence-Based Bound}

\begin{corollary}[
\Cref{thm:viallard_smv_DDBC} with Dirichlet distributions
]
\label{cor:viallard_smv_DDBC_dirichlet}
For any distribution $\Dcal$, finite hypothesis sets $\Hcalone$ and $\Hcaltwo$, majority votes weight spaces $\Wcalone = \{\brho\in[0,1]^{|\!\Hcalone\!|}:\|\brho\|_1=1\}$, $\Wcaltwo = \{\brho\in[0,1]^{|\!\Hcaltwo\!|}:\|\brho\|_1=1\}$, hyper-prior distributions $\hyperpriorone=\Dir(\bbeta^{(1)})$, $\hyperpriortwo=\Dir(\bbeta^{(2)})$, loss $\ell:\widehat{\Ycal}\times\Ycal\to[0,1]$, algorithms $\Aone:(\Xcal{\times}\Ycal)^\none \times \Mcal^*(\Wcalone )\rightarrow \Mcal(\Wcalone )$ and $\Atwo :(\Xcal{\times}\Ycal)^\ntwo \times \Mcal^*(\Wcaltwo)\rightarrow \Mcal(\Wcaltwo)$, $\lambda>1$ and $\delta\in (0,1]$, with $\balpha^{(1)},\bbeta^{(1)}\in\Rbb_{>0}^{|\!\Hcalone\!|}$, $\balpha^{(2)},\bbeta^{(2)} \in\Rbb_{>0}^{|\!\Hcaltwo\!|}$, $\none=\ntwo=n/2$ and $\tau=1/2$, we have 
\begin{equation*}
\Prob_{\substack{
S\sim\Dcal^n\\
\posteriorone\sim\Dir(\balpha^{(1)})\\
\posteriortwo\sim\Dir(\balpha^{(2)})
}}
\left[
\begin{aligned}
&\kl\left(
\frac12\widehat R_{\Sone}(\MV_{\posteriorone})
+
\frac12\widehat R_{\Stwo}(\MV_{\posteriortwo})
\middle\|
\frac12\Risktrue(\MV_{\posteriorone})
+
\frac12\Risktrue(\MV_{\posteriortwo})
\right)
\\
&\leq
\frac{1}{n}
\Bigg[
2\frac{2\lambda-1}{\lambda-1}\ln\frac{4}{\delta}
+
\ln(2n)
+
\ln\frac{B(\bbeta^{(1)})B(\bbeta^{(2)})}
{B(\balpha^{(1)})B(\balpha^{(2)})}
\\
&\qquad+
\frac{1}{\lambda-1}
\ln
\frac{
B\left(\lambda\balpha^{(1)}+(1-\lambda)\bbeta^{(1)}\right)
B\left(\lambda\balpha^{(2)}+(1-\lambda)\bbeta^{(2)}\right)
}{
B(\balpha^{(1)})B(\balpha^{(2)})
}
\Bigg]
\end{aligned}
\right]
\geq 1-\delta.
\end{equation*}
where $\hyperposterior_{\Sone } \triangleq \Aone(\Sone , \hyperpriorone)$ and $\hyperposterior_{\Stwo } \triangleq \Atwo(\Stwo , \hyperpriortwo)$.
\end{corollary}

\begin{proof}
Apply \Cref{eq:Rényi_dir} to the pairs
$(\hyperposterior_{\Sone},\hyperpriorone)$ and
$(\hyperposterior_{\Stwo},\hyperpriortwo)$ in
\Cref{thm:viallard_smv_DDBC}. Since
\[
\frac{\tau}{\none}
=
\frac{1-\tau}{\ntwo}
=
\frac{1}{n},
\qquad
\ln(2\sqrt{\none})+\ln(2\sqrt{\ntwo})
=
\ln(2n),
\]
the result follows.
\end{proof}

The corresponding self-bounding optimization objective is
\begin{equation}
\begin{aligned}
\kl^{-1}\Bigg(
&\frac12\widehat R_{\Sone}(\MV_{\posteriorone})
+
\frac12\widehat R_{\Stwo}(\MV_{\posteriortwo})
\;\Bigg\|\;
\frac{1}{n}
\Bigg[
2\frac{2\lambda-1}{\lambda-1}\ln\frac{4}{\delta}
+
\ln(2n)
+
\ln\frac{B(\bbeta^{(1)})B(\bbeta^{(2)})}
{B(\balpha^{(1)})B(\balpha^{(2)})}
\\
&\qquad+
\frac{1}{\lambda-1}
\ln
\frac{
B\left(\lambda\balpha^{(1)}+(1-\lambda)\bbeta^{(1)}\right)
B\left(\lambda\balpha^{(2)}+(1-\lambda)\bbeta^{(2)}\right)
}{
B(\balpha^{(1)})B(\balpha^{(2)})
}
\Bigg]
\Bigg).
\end{aligned}
\label{eq:obj_viallard_DDBC_dir}
\end{equation}

\section{Experimental Details}\label{app:exp}

\subsection{Datasets}\label{app:datasets}

We consider nine binary and six multiclass classification datasets
obtained from UCI~\citep{dua2017uci},
\href{https://www.csie.ntu.edu.tw/~cjlin/libsvm/}{LIBSVM},
and Zalando~\citep{xiao2017fashion}.
Tables~\ref{tab:binary-datasets} and~\ref{tab:multiclass-datasets}
report the number of instances, the number of input features after
preprocessing, the number of classes, and the associated prediction task.

\begin{table}[ht!]
\centering
\caption{Binary classification datasets.}
\label{tab:binary-datasets}
\small
\begin{tabular}{llrrp{6cm}}
\toprule
Dataset & Source & $n$ & $d$ & Prediction task \\
\midrule
Haberman
& UCI & 306 & 3
& Survival after surgery. \\

TicTacToe
& UCI & 958 & 9
& Winning configurations for player~$x$. \\

Svmguide1
& LIBSVM & 7,089 & 4
& Binary classification without further description. \\

Mushrooms
& UCI & 8,124 & 22
& Edibility from categorical appearance features. \\

Phishing
& LIBSVM & 2,456 & 68
& Detection of phishing websites. \\

Adult
& LIBSVM & 32,561 & 123
& Prediction of whether annual income exceeds \$50K. \\

CodRNA
& LIBSVM & 59,535 & 8
& Detection of non-coding RNAs. \\

Splice
& UCI & 3,190 & 60
& Prediction of DNA splice junctions. \\

Australian
& LIBSVM & 690 & 14
& Credit approval prediction. \\
\bottomrule
\end{tabular}
\end{table}

\begin{table}[H]
\centering
\caption{Multiclass classification datasets.}
\label{tab:multiclass-datasets}
\small
\begin{tabular}{llrrrp{5cm}}
\toprule
Dataset & Source & $n$ & $d$ & Classes & Prediction task \\
\midrule
Pendigits
& UCI & 10,992 & 16 & 10
& Recognition of handwritten digits. \\

Protein
& LIBSVM & 24,387 & 357 & 3
& Classification of protein structures. \\

Shuttle
& UCI & 58,000 & 9 & 7
& Classification of space-shuttle conditions. \\

Sensorless
& LIBSVM & 58,509 & 48 & 11
& Prediction of motor operating conditions. \\

MNIST
& LIBSVM & 70,000 & 784 & 10
& Recognition of handwritten digits. \\

Fashion-MNIST
& Zalando & 70,000 & 784 & 10
& Recognition of clothing articles. \\
\bottomrule
\end{tabular}
\end{table}

\subsection{Complete Numerical Results}\label{subsec:numerical_results}

\Cref{tab:best_bounds_binary,tab:best_test_error_binary,tab:best_time_binary}
report the bound values, test errors, and training times, respectively,
on the binary datasets. The corresponding results on the multiclass
datasets are reported in
\Cref{tab:best_bounds_multiclass,tab:best_test_error_multiclass,tab:best_time_multiclass}.
All results are averaged over 10 independent runs and reported with
their standard deviations. Bound values and test errors are expressed
as percentages, whereas training times are expressed in seconds.

\begin{table*}[ht!]
\centering
\scriptsize
\setlength{\tabcolsep}{2pt}
\resizebox{\textwidth}{!}{%
\begin{tabular}{lccccccccc}
\toprule
Method & MUSHROOMS & TICTACTOE & SVMGUIDE & HABERMAN & PHISHING & CODRNA & ADULT & SPLICE & AUSTRALIAN \\
\midrule
FO & $12.30 \pm 0.19$ & $76.47 \pm 1.45$ & $17.81 \pm 0.21$ & $75.73 \pm 1.73$ & $25.73 \pm 0.31$ & $49.45 \pm 0.47$ & $46.93 \pm 0.18$ & $54.67 \pm 0.93$ & $43.32 \pm 1.19$ \\
SO & $21.02 \pm 0.24$ & $100.20 \pm 1.02$ & $28.56 \pm 0.27$ & $110.58 \pm 1.42$ & $35.64 \pm 0.22$ & $60.55 \pm 0.21$ & $53.27 \pm 0.09$ & $63.65 \pm 0.48$ & $69.46 \pm 1.23$ \\
Bin & $11.44 \pm 0.27$ & $79.50 \pm 0.88$ & $22.30 \pm 0.27$ & $80.36 \pm 0.73$ & $24.63 \pm 0.20$ & $36.20 \pm 0.13$ & $41.04 \pm 0.13$ & $46.40 \pm 0.44$ & $54.45 \pm 0.93$ \\
SMV-Exact & $\mathbf{4.68 \pm 0.09}$ & $\mathbf{42.74 \pm 0.63}$ & $\mathbf{9.26 \pm 0.15}$ & $\mathbf{42.17 \pm 0.84}$ & $\mathbf{13.56 \pm 0.11}$ & $\mathbf{15.90 \pm 2.99}$ & $\mathbf{24.82 \pm 1.73}$ & $32.07 \pm 22.65$ & $\mathbf{28.57 \pm 0.56}$ \\
SMV-MC & $5.34 \pm 0.12$ & $43.36 \pm 0.70$ & $9.44 \pm 0.15$ & $44.26 \pm 1.52$ & $15.25 \pm 0.12$ & $17.15 \pm 2.59$ & $26.74 \pm 0.02$ & $35.82 \pm 21.40$ & $29.47 \pm 0.76$ \\
DIS-V & $6.60 \pm 1.15$ & $48.69 \pm 1.32$ & $12.40 \pm 3.87$ & $54.77 \pm 4.79$ & $16.47 \pm 0.76$ & $20.18 \pm 4.38$ & $27.32 \pm 0.28$ & $31.72 \pm 1.23$ & $35.91 \pm 2.05$ \\
DIS-R & $5.95 \pm 1.27$ & $45.58 \pm 2.63$ & $9.80 \pm 0.84$ & $51.45 \pm 7.68$ & $15.26 \pm 0.33$ & $18.16 \pm 3.55$ & $26.79 \pm 0.25$ & $\mathbf{29.25 \pm 1.32}$ & $33.00 \pm 3.09$ \\
\bottomrule
\end{tabular}%
}
\caption{Mean bound value and standard deviation over 10 independent runs on binary datasets.}
\label{tab:best_bounds_binary}
\end{table*}

\begin{table*}[ht!]
\centering
\scriptsize
\setlength{\tabcolsep}{2pt}
\resizebox{\textwidth}{!}{%
\begin{tabular}{lccccccccc}
\toprule
Method & MUSHROOMS & TICTACTOE & SVMGUIDE & HABERMAN & PHISHING & CODRNA & ADULT & SPLICE & AUSTRALIAN \\
\midrule
FO & $4.60 \pm 0.36$ & $29.22 \pm 2.77$ & $6.26 \pm 0.37$ & $25.48 \pm 3.21$ & $11.26 \pm 0.58$ & $23.52 \pm 0.35$ & $22.14 \pm 0.36$ & $22.28 \pm 1.80$ & $16.09 \pm 2.04$ \\
SO & $4.60 \pm 0.36$ & $30.00 \pm 3.84$ & $6.26 \pm 0.37$ & $\mathbf{24.35 \pm 3.34}$ & $10.14 \pm 0.70$ & $22.93 \pm 1.43$ & $16.96 \pm 0.30$ & $12.79 \pm 1.33$ & $16.09 \pm 2.04$ \\
Bin & $\mathbf{1.08 \pm 0.22}$ & $\mathbf{27.71 \pm 3.60}$ & $5.29 \pm 0.45$ & $26.45 \pm 1.29$ & $\mathbf{6.82 \pm 0.41}$ & $\mathbf{12.38 \pm 0.25}$ & $\mathbf{16.21 \pm 0.36}$ & $\mathbf{7.10 \pm 1.05}$ & $\mathbf{15.94 \pm 2.15}$ \\
SMV-Exact & $1.34 \pm 0.21$ & $30.75 \pm 2.31$ & $\mathbf{5.08 \pm 0.34}$ & $26.67 \pm 2.75$ & $8.17 \pm 0.35$ & $12.83 \pm 3.53$ & $22.25 \pm 2.97$ & $14.39 \pm 12.31$ & $16.99 \pm 1.91$ \\
SMV-MC & $1.27 \pm 0.23$ & $30.57 \pm 2.30$ & $5.11 \pm 0.41$ & $27.10 \pm 1.85$ & $7.98 \pm 0.40$ & $13.39 \pm 3.36$ & $24.07 \pm 0.00$ & $14.46 \pm 12.30$ & $16.56 \pm 2.01$ \\
DIS-V & $1.98 \pm 1.46$ & $29.90 \pm 3.74$ & $6.80 \pm 2.91$ & $28.55 \pm 5.86$ & $8.22 \pm 0.67$ & $16.87 \pm 5.34$ & $24.07 \pm 0.00$ & $10.61 \pm 1.66$ & $16.23 \pm 2.13$ \\
DIS-R & $1.93 \pm 1.33$ & $29.90 \pm 2.88$ & $5.43 \pm 0.66$ & $26.77 \pm 1.07$ & $7.93 \pm 0.65$ & $14.67 \pm 4.24$ & $24.07 \pm 0.00$ & $10.09 \pm 1.48$ & $17.03 \pm 2.49$ \\

\bottomrule
\end{tabular}%
}
\caption{Mean test error and standard deviation over 10 independent runs on binary datasets.}
\label{tab:best_test_error_binary}
\end{table*}

\begin{table*}[ht!]
\centering
\scriptsize
\setlength{\tabcolsep}{3pt}
\begin{tabular}{lcccccc}
\toprule
Method & PENDIGITS & SHUTTLE & SENSORLESS & PROTEIN & MNIST & FASHION-MNIST \\
\midrule
FO & $21.69 \pm 0.30$ & $0.36 \pm 0.02$ & $2.19 \pm 0.18$ & $112.64 \pm 0.47$ & $47.16 \pm 0.30$ & $53.20 \pm 0.26$ \\
SO & $17.41 \pm 0.20$ & $0.47 \pm 0.02$ & $1.95 \pm 0.04$ & $138.57 \pm 0.37$ & $44.81 \pm 0.18$ & $58.69 \pm 0.19$ \\
Bin & $9.01 \pm 0.28$ & $0.30 \pm 0.02$ & $1.05 \pm 0.06$ & $127.29 \pm 3.02$ & $30.78 \pm 0.22$ & $44.25 \pm 0.24$ \\
SMV-Exact & $\mathbf{4.80 \pm 0.13}$ & $\mathbf{0.16 \pm 0.01}$ & $\mathbf{0.55 \pm 0.03}$ & $59.00 \pm 0.72$ & $\mathbf{15.94 \pm 0.11}$ & $\mathbf{22.56 \pm 0.12}$ \\
SMV-MC & $4.86 \pm 0.16$ & $\mathbf{0.16 \pm 0.01}$ & $\mathbf{0.55 \pm 0.03}$ & $\mathbf{58.77 \pm 0.41}$ & $15.98 \pm 0.11$ & $\mathbf{22.56 \pm 0.12}$ \\
DIS-V & $5.97 \pm 0.22$ & $0.27 \pm 0.02$ & $0.82 \pm 0.06$ & $60.27 \pm 0.42$ & $16.60 \pm 0.16$ & $23.26 \pm 0.18$ \\
DIS-R & $5.10 \pm 0.13$ & $0.18 \pm 0.02$ & $0.58 \pm 0.03$ & $59.17 \pm 0.50$ & $16.04 \pm 0.15$ & $22.61 \pm 0.16$ \\
\bottomrule
\end{tabular}
\caption{Mean bound value and standard deviation over 10 independent runs on multiclass datasets.}
\label{tab:best_bounds_multiclass}
\end{table*}

\begin{table*}[ht!]
\centering
\scriptsize
\setlength{\tabcolsep}{3pt}
\begin{tabular}{lcccccc}
\toprule
Method & PENDIGITS & SHUTTLE & SENSORLESS & PROTEIN & MNIST & FASHION-MNIST \\
\midrule
FO & $8.26 \pm 1.07$ & $0.07 \pm 0.02$ & $0.68 \pm 0.11$ & $\mathbf{54.37 \pm 0.59}$ & $22.41 \pm 0.33$ & $25.31 \pm 0.39$ \\
SO & $\mathbf{3.00 \pm 0.19}$ & $\mathbf{0.04 \pm 0.02}$ & $\mathbf{0.20 \pm 0.06}$ & $63.47 \pm 0.38$ & $\mathbf{14.34 \pm 0.17}$ & $\mathbf{21.11 \pm 0.38}$ \\
Bin & $3.18 \pm 0.19$ & $0.06 \pm 0.02$ & $0.21 \pm 0.06$ & $59.33 \pm 4.85$ & $14.48 \pm 0.19$ & $21.28 \pm 0.37$ \\
SMV-Exact & $3.44 \pm 0.22$ & $0.06 \pm 0.02$ & $0.23 \pm 0.05$ & $54.79 \pm 0.50$ & $14.87 \pm 0.18$ & $21.43 \pm 0.32$ \\
SMV-MC & $3.44 \pm 0.22$ & $0.06 \pm 0.02$ & $0.23 \pm 0.05$ & $54.91 \pm 0.46$ & $14.87 \pm 0.18$ & $21.43 \pm 0.32$ \\
DIS-V & $3.43 \pm 0.27$ & $0.06 \pm 0.02$ & $0.23 \pm 0.06$ & $55.00 \pm 0.46$ & $14.81 \pm 0.20$ & $21.38 \pm 0.29$ \\
DIS-R & $3.42 \pm 0.22$ & $0.06 \pm 0.02$ & $0.22 \pm 0.05$ & $55.43 \pm 0.69$ & $14.80 \pm 0.20$ & $21.35 \pm 0.30$ \\
\bottomrule
\end{tabular}
\caption{Mean test error and standard deviation over 10 independent runs on multiclass datasets.}
\label{tab:best_test_error_multiclass}
\end{table*}

\begin{table*}[ht!]
\centering
\scriptsize
\setlength{\tabcolsep}{2pt}
\resizebox{\textwidth}{!}{%
\begin{tabular}{lccccccccc}
\toprule
Method & MUSHROOMS & TICTACTOE & SVMGUIDE & HABERMAN & PHISHING & CODRNA & ADULT & SPLICE & AUSTRALIAN \\
\midrule
FO & $142.89 \pm 4.21$ & $50.95 \pm 6.23$ & $91.66 \pm 7.94$ & $47.74 \pm 2.69$ & $215.81 \pm 8.64$ & $2046.86 \pm 47.46$ & $1194.02 \pm 74.84$ & $58.22 \pm 8.35$ & $49.79 \pm 1.32$ \\
SO & $\mathbf{115.51 \pm 2.09}$ & $43.52 \pm 8.90$ & $87.70 \pm 8.83$ & $43.10 \pm 5.26$ & $201.53 \pm 4.41$ & $1858.73 \pm 34.27$ & $967.94 \pm 17.29$ & $\mathbf{55.98 \pm 10.25}$ & $\mathbf{41.63 \pm 8.78}$ \\
Bin & $386.12 \pm 73.43$ & $94.37 \pm 16.00$ & $399.55 \pm 1.58$ & $63.11 \pm 6.19$ & $595.18 \pm 2.01$ & $3365.27 \pm 20.95$ & $1456.36 \pm 4.35$ & $100.01 \pm 7.91$ & $86.93 \pm 17.93$ \\
SMV-Exact & $335.36 \pm 71.96$ & $101.23 \pm 1.35$ & $357.85 \pm 54.96$ & $62.91 \pm 4.64$ & $560.97 \pm 2.71$ & $3311.68 \pm 322.52$ & $1795.34 \pm 238.56$ & $122.14 \pm 24.57$ & $62.81 \pm 12.50$ \\
SMV-MC & $155.59 \pm 0.67$ & $59.40 \pm 1.18$ & $113.53 \pm 0.96$ & $51.32 \pm 0.38$ & $283.56 \pm 0.84$ & $857.40 \pm 349.01$ & $1159.25 \pm 5.17$ & $80.51 \pm 10.89$ & $49.92 \pm 0.73$ \\
DIS-V & $129.05 \pm 1.03$ & $54.53 \pm 0.63$ & $100.60 \pm 1.00$ & $48.93 \pm 0.35$ & $\mathbf{153.58 \pm 0.85}$ & $782.13 \pm 296.39$ & $\mathbf{437.99 \pm 2.54}$ & $65.77 \pm 0.86$ & $48.88 \pm 0.89$ \\
DIS-R & $455.16 \pm 6.81$ & $\mathbf{23.73 \pm 0.24}$ & $\mathbf{59.73 \pm 0.93}$ & $\mathbf{17.59 \pm 0.21}$ & $643.22 \pm 26.96$ & $\mathbf{548.73 \pm 211.72}$ & $602.67 \pm 9.85$ & $159.15 \pm 25.63$ & $50.71 \pm 0.70$ \\
\bottomrule
\end{tabular}%
}
\caption{Mean training time (s) and standard deviation over 10 independent runs on binary datasets.}
\label{tab:best_time_binary}
\end{table*}

\begin{table*}[ht!]
\centering
\scriptsize
\setlength{\tabcolsep}{3pt}
\begin{tabular}{lcccccc}
\toprule
Method & PENDIGITS & SHUTTLE & SENSORLESS & PROTEIN & MNIST & FASHION-MNIST \\
\midrule
FO & $226.80 \pm 29.27$ & $646.61 \pm 163.66$ & $362.61 \pm 66.34$ & $596.66 \pm 2.56$ & $2250.35 \pm 64.33$ & $1797.51 \pm 10.49$ \\
SO & $\mathbf{203.51 \pm 41.49}$ & $\mathbf{429.37 \pm 119.38}$ & $\mathbf{305.80 \pm 72.95}$ & $620.86 \pm 5.42$ & $2166.82 \pm 13.15$ & $1715.14 \pm 5.82$ \\
Bin & $565.99 \pm 146.37$ & $1751.15 \pm 563.66$ & $894.94 \pm 205.56$ & $1486.81 \pm 323.81$ & $5442.73 \pm 70.23$ & $4855.80 \pm 39.75$ \\
SMV-Exact & $617.51 \pm 69.39$ & $1888.17 \pm 342.51$ & $915.53 \pm 233.45$ & $1542.98 \pm 155.46$ & $5043.08 \pm 13.02$ & $4590.26 \pm 21.46$ \\
SMV-MC & $259.14 \pm 3.09$ & $826.03 \pm 3.38$ & $438.14 \pm 2.19$ & $633.41 \pm 3.73$ & $2263.86 \pm 24.81$ & $1845.17 \pm 6.89$ \\
DIS-V & $246.13 \pm 4.09$ & $768.09 \pm 15.07$ & $413.05 \pm 2.72$ & $\mathbf{588.34 \pm 2.56}$ & $\mathbf{2159.57 \pm 20.52}$ & $\mathbf{1657.76 \pm 9.32}$ \\
DIS-R & $240.02 \pm 3.07$ & $763.88 \pm 8.97$ & $458.38 \pm 7.54$ & $692.62 \pm 8.65$ & $2337.83 \pm 44.63$ & $1785.85 \pm 10.04$ \\

\bottomrule
\end{tabular}
\caption{Mean training time (s) and standard deviation over 10 independent runs on multiclass datasets.}
\label{tab:best_time_multiclass}
\end{table*}

\end{document}